\documentclass[letterpaper]{article} 
\usepackage{aaai2026}  
\usepackage{times}  
\usepackage{helvet}  
\usepackage{courier}  
\usepackage[hyphens]{url}  
\usepackage{graphicx} 
\def\UrlFont{\rm}  
\usepackage{natbib}  
\usepackage{caption} 
\usepackage{amsmath}
\usepackage{amssymb} 
\usepackage{booktabs}
\usepackage{cuted}

\DeclareMathOperator*{\E}{\mathbb{E}}

\def\A{\mathcal{A}}
\def\G{\mathcal{G}}
\def\H{\mathcal{H}}
\def\S{\mathcal{S}}

\def\ld{\log_2}
\def\entropy{\mathbb{H}} 
\def\MI{\mathbb{I}} 
\def\DKL{D_{KL}} 

\usepackage{todonotes}
\usepackage{xcolor}
\usepackage{url}

\usepackage{amsthm}
\newtheorem{proposition}{Proposition}

\nocopyright 

\title{Model-Based Soft Maximization of Suitable Metrics of Long-Term Human Power}
\author{
    Jobst Heitzig\textsuperscript{\rm 1},
    Ram Potham\textsuperscript{\rm 2}
}
\affiliations{
    \textsuperscript{\rm 1}Potsdam Institute for Climate Impact Research, Potsdam, Germany, heitzig@pik-potsdam.de\\
    \textsuperscript{\rm 2}Independent, ram.potham@gmail.com

}

\begin{document}

\maketitle

\begin{abstract}
Power is a key concept in AI safety: power-seeking as an instrumental goal, sudden or gradual disempowerment of humans, power balance in human-AI interaction and international AI governance. 
At the same time, power as the ability to pursue diverse goals is essential for wellbeing.

This paper explores the idea of promoting both safety and wellbeing by forcing AI agents explicitly to empower humans and to manage the power balance between humans and AI agents in a desirable way.
Using a principled, partially axiomatic approach, we design a parametrizable and decomposable objective function that represents an inequality- and risk-averse long-term aggregate of human power. 
It takes into account humans' bounded rationality and social norms, and, crucially, considers a wide variety of possible human goals.

We derive algorithms for computing that metric by backward induction or approximating it via a form of multi-agent reinforcement learning from a given world model.
We exemplify the consequences of (softly) maximizing this metric in a variety of paradigmatic situations and describe what instrumental sub-goals it will likely imply.
Our cautious assessment is that softly maximizing suitable aggregate metrics of human power might constitute a beneficial objective for agentic AI systems that is safer than direct utility-based objectives. 
\end{abstract}


\section{Introduction}

A duty to empower others, especially those with little power, can be defended on consequentialist \cite{sen2014development}, deontological \cite{hill2002human}, and virtue ethical \cite{nussbaum2019aristotelian} grounds.
At the same time, gradual or sudden {\em dis}empowerment of humans due to misaligned A(G)I, e.g.\ due to power-seeking \cite{turner2019optimal} or non-corrigible over-optimization of misaligned reward functions \cite{gao2023scaling}, is a key AI safety risk.

Accordingly, some papers explored tasking AI agents explicitly with the empowerment of {\em individual} humans, mainly inspired by the information-theoretic channel capacity $\cal E$ between human actions and environmental states, called `empowerment' in \citet{klyubin2005empowerment}.
But that metric is hard to compute, so \citet{du2020ave} use a proxy based on the variance of terminal states reached under a random policy, while \citet{myers2024learning} use contrastive learning of latent state representations to approximate $\cal E$.
\citet{salge2017empowerment} discuss the `empowerment'-based approach in more depth for the single-human case.

This paper extends the theoretical foundations of this approach in two ways.
First, we develop an {\em alternative metric} of human power, `ICCEA power', that is directly based on the key aspect of power: the ability to attain a wide range of possible goals. 
Our metric explicitly and transparently incorporates humans' knowledge about the actions of the AI agent, their expectations about others' behavior, e.g.\ due to social norms, and their bounded rationality. 
Second, using an approach guided by desiderata similar to the axioms of social choice and welfare theory, we develop an objective function for an AI agent interacting with {\em populations} of humans, based on an aggregate of ICCEA power across humans and time that gives the agent desirable incentives such as communicating well \cite{reddy2022first}, following orders, being corrigible \cite{potham2025corrigibility}, avoiding irreversible changes in the environment, protecting humans and itself from harm and disempowerment, allocating resources fairly and sustainably, and acting ``appropriately'' by following relevant social norms \cite{leibo2024theory}.

Though based on {\em possible} human goals, our approach avoids trying to learn individuals' {\em actual, current} goals, because human preferences are changing and non-identifiable \cite{cao2021identifiability,banerjee2011poor} and their prediction is unavoidably uncertain \cite{baker2011bayesian}.
Some critics of a preference-based approach argue for a {\em values}-based approach instead \cite{lowe2025fullstack}, which however requires an even more {\em semantic} world understanding by the AI agent.
A deep semantic understanding is also required in the `freedoms'-based conception of AI ethics in \cite{london2024beneficent}, and compiling its required list of `fundamental capabilities' is difficult \cite{robeyns2006capability}.
By contrast, like the `empowerment'-based approach, our metrics are mostly based on a {\em structural} understanding of possibly dynamics, interactions, and transition probabilities and aim to avoid semantic issues by relating human power to the ability to bring about just {\em any} possible conditions a human might happen to desire.




For theoretical convenience, we work in a {\em model-based} setting where the AI agent can plan on the basis of a decent stochastic world model like the ``scientist AI'' envisioned in \citet{bengio2025superintelligent}.  
After developing our human power metric in Section \ref{sec:human_power_metric}, we shortly describe algorithms for softly maximizing it in Section \ref{sec:maximize_human_power}, before reporting insights about the resulting behavior in Section \ref{sec:experiments}. Section 5 concludes.



\section{Measuring, Aggregating, and Softly Maximizing Human (Em)Power(ment)}
\label{sec:human_power_metric}

\begin{figure}[ht]
  \centering
  \includegraphics[width=0.95\columnwidth]{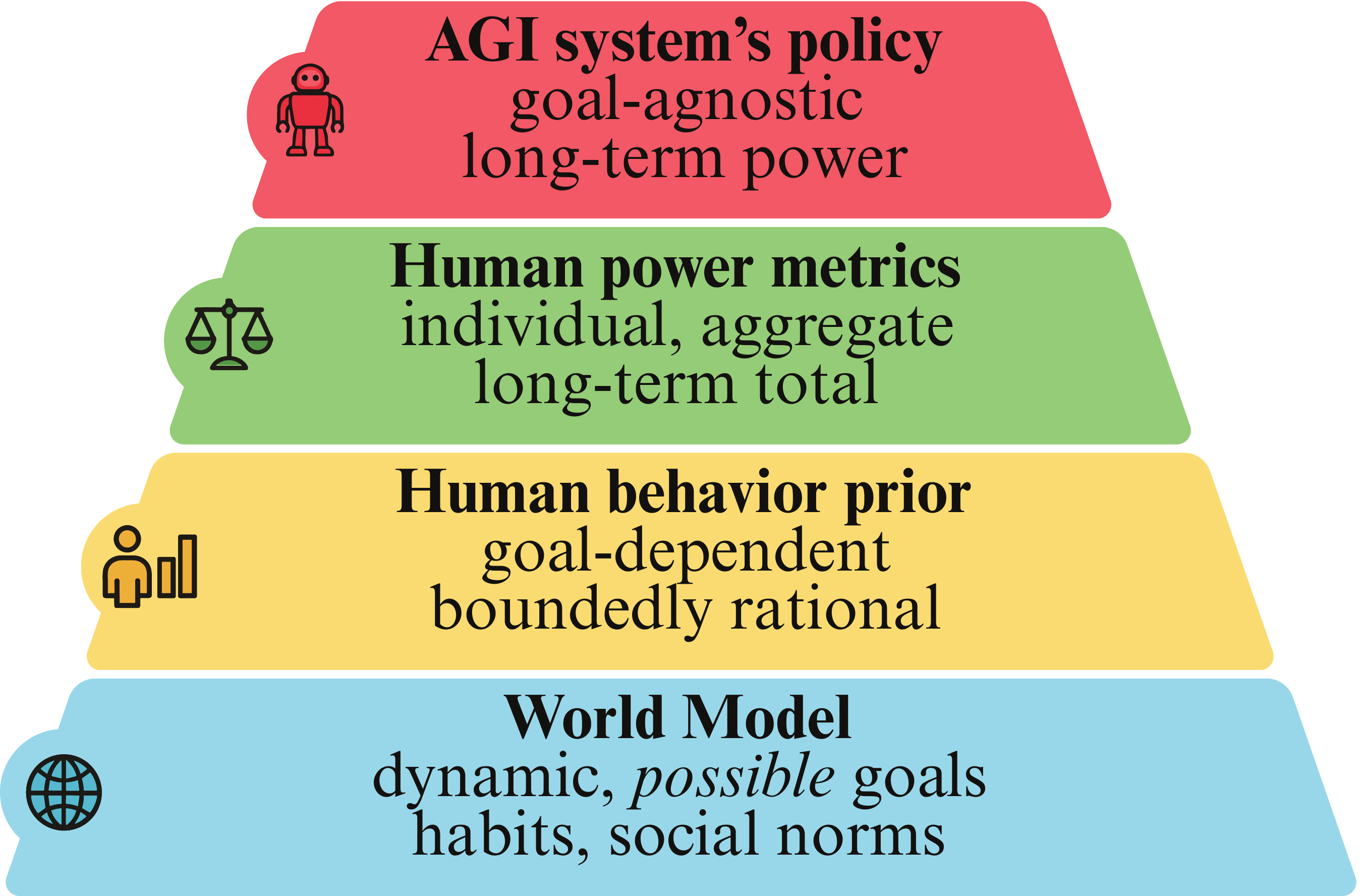}
  \caption{Overview of proposed approach for deriving power-managing policies for a general-purpose AGI system}
  \label{fig:main_figure}
\end{figure}

\paragraph{Framework}
We assume a \textbf{\textit{robot} \boldmath $r$} interacting with several \textbf{\textit{humans} \boldmath $h\in\H$}. 
The robot models the world as a (fully or partially observed) stochastic game with states $s\in\S$.
It gets no {\em extrinsic} reward itself.
Humans have many {\em possible} and potentially {\em changing} \textbf{\textit{goals} \boldmath $g_h\in\G_h$} and get goal-dependent rewards $U_h(s',g_h)$, discounted at factors $\gamma_h<1$. 
Action sets $\A_r(s),\A_h(s)$ may be state-dependent.
Action and policy profiles are written $a=(a_r,a_\H)=(a_r,a_h,a_{-h})$,
$\pi=(\pi_r,\pi_\H)$.
The transition kernel is $P(s'|s,a)$. 

Crucially, we assume the robot neither knows nor forms beliefs about the {\em actual} goals of the humans. 
It {\em does} make assumptions about their level of rationality and their beliefs about each others' behavior, reflected in additional model parameters $\nu_h$, $\pi^0_h$, $\beta_h$, and $\mu_{-h}$, as detailed below.

\paragraph{Task} We want to design an algorithm for computing a policy $\pi_r$ for the robot that implements the aggregate human power maximization objective. 
For tractability, we formalize this objective as an expected discounted return, $V_r = \E_{s_{\ge 0}}\sum_{t\ge 0}\gamma_r^t U_r(s_t)$, based on an {\em intrinsic} reward $U_r(s)$ which represents a suitable assessment of the aggregate human power in $s$.
To design $U_r(s)$, we first design a power metric $W_h(s)$ for individual humans, and then an aggregation function $U_r(s) = F_U((W_h(s))_{h\in\H})$. 

In designing $W_h$ and $F_U$, we have several objectives relating to their tractability, interpretability, and the behavioral incentives that they give the robot. 

\begin{table}[!b]
    \noindent\rule{\linewidth}{0.8pt}\\[-4ex]
    \begin{align}
        Q^m_h(s,g_h,a_h) &\gets \textstyle\E_{a_{-h}\sim\mu_{-h}(s,g_h)} \min_{a_r\in\A_r(s)} \E_{s'\sim s,a} \nonumber\\
            &\qquad\quad \big(U_h(s',g_h) + \gamma_h V^m_h(s',g_h)\big), \label{Qm} \\
        \pi_h(s,g_h) &\gets\nu_h(s,g_h)\pi^0_h(s,g_h) + (1-\nu_h(s,g_h))\times{}\nonumber\\
            &~~{}\times\text{$\beta_h(s,g_h)$-softmax for~} Q^m_h(s,g_h,\cdot), \label{pih} \\ 
        V^m_h(s,g_h) &\gets \textstyle\E_{a_h\sim\pi_h(s,g_h)} Q^m_h(s,g_h,a_h), \label{Vm} \\[-3ex]
\intertext{\noindent\rule{\linewidth}{0.4pt}} \nonumber\\[-5ex]
        Q_r(s,a_r) &\gets \textstyle \E_g\E_{a_\H\sim\pi_\H(s,g)} \E_{s'\sim s,a} \gamma_r V_r(s'), \label{Qr} \\
        \pi_r(s)(a) &\propto (-Q_r(s,a_r))^{-\beta_r}, \label{pir} \\
        V^e_h(s,g_h) &\gets \textstyle\E_{g_{-h}}\E_{a_{\H}\sim\pi_{\H}(s,g)} \E_{a_r\sim\pi_r(s)} \E_{s'\sim s,a}  \nonumber\\
            &\qquad\qquad \textstyle\big(U_h(s',g_h) + \gamma_h V^e_h(s',g_h)\big), \label{Ve} \\
        X_h(s) &\gets \textstyle \sum_{g_h\in\G_h} V^e_h(s,g_h)^\zeta, \label{Xh} \\
        U_r(s) &\gets \textstyle -\big(\sum_{h} X_h(s)^{-\xi}\big)^\eta, \label{Ur} \\
        V_r(s) &\gets \textstyle U_r(s) + \E_{a_r\sim\pi_r(s)} Q_r(s,a_r). \label{Vr}
    \end{align}\\[-5ex]
    \rule{\linewidth}{0.8pt}
    \caption{Computation of ICCEA power $W_h=\ld X_h$ and intrinsic robot reward $U_r$ as derived in the text.}
    \label{tab:equations}
\end{table}

\begin{table*}[ht]
    \centering
    \begin{tabular}{ll}
       \toprule
       \bf Desideratum  & \bf Corresponding design choice or world model requirement \\
       \midrule
       $V_h$ recursively computable & Define via Bellman equation: $V_h(s,g_h)=\E_{s'}(U_h(s',g_h)+\gamma_h V_h(s',g_h))$ \\
       $V_h$ comparable across humans
       & Assume goals $g_h$ are events: $g_h\subseteq\S$, $U_h(s',g_h)=1_{s'\in g_h}$ \\
       $V_h$ has natural interpretation & Make $V_h$ a goal-reaching probability: $s\in g_h$, $s'$ reachable from $s\Rightarrow s'\notin g_h$ \\
       \midrule 
       $r$ incentivized to make commitments & Base $\pi_h$ on $h$ assuming $r$ takes worst action not ruled out by $r$'s commitments \\
       $r$ considers $h$'s bounded rationality & Base $\pi_h$ on varying, unstable goals, habits, social norms, mutual expectations \\
       \midrule
       $W_h$ based on $r$'s best estimate of $V_h$ & Distinguish $h$'s simulated estimate $V^m_h$ for $\pi_h$ and $r$'s estimate $V^e_h$ for $W_h$ \\
       $W_h$ indep.~of ``unaffected'' goals & Use separable ansatz: $W_h = F^G(\sum_{g_h} f^G(V^e_h))$  ($\Rightarrow$ enables stoch.~approx.) \\
       $W_h$ additive across indep.~subgames & $F^G=\log_2$ ($\Rightarrow W_h=$ ``certainty-equivalent'' effective no.~of binary choices) \\
       $W_h>-\infty$ & Make set $\G_h$ of possible goals wide enough to cover all possible trajectories \\
       Range of $W_h$ is symmetric around 0 & Make each trajectory fulfill exactly one $g_h\in\G_h$ and put $\zeta=2$ (see below) \\ 
       \midrule
       $U_r$ indep.~of ``unconcerned'' agents & Use separable ansatz: $U_r = F^H(\sum_h f^H(W_h))$ ($\Rightarrow$ enables stoch.~approx.) \\
       Pigou--Dalton-type inequality aversion & Make $f^H$ strictly concave, e.g., $f^H(w)=-2^{-\xi w}$ with $\xi>0$ \\
       Protect a human's ``last'' bit of power & Choose $\xi\ge 1$, e.g., $\xi=1$ \\
       $r$ cares for current {\em and} later human power & Base $\pi_r$ on $V_r(s_0) = \E_{s_{\ge 0}} \sum_{t=0}^\infty \gamma_r^t U_r(s_t)$, not just $U_r(s_0)$ \\  
       Limit intertemporal power trading & Also make $F^H$ strictly concave (intertemporal inequality aversion) \\ 
       \midrule
       $\pi_r$ indep.~of common rescaling of $V^e_h$ & Use power laws: $f^G(v)=v^\zeta$, $\zeta>0$; $f^H(w)=-2^{-\xi w}$, $\xi\ge 1$; \\
       & $F^H(y)=-(-y)^\eta$, $\eta > 1$; and $\pi(s)(a)\propto (-Q_r(s,a_r))^{-\beta_r}$, $0\le\beta_r\le\infty$ \\
       $r$ incentivized to reduce uncertainty & Choose $\zeta>1$ (risk aversion, preference for reliability) \\
       Avoid risks from over-optimization & Choose $\beta_r<\infty$ (soft optimization, exploration) \\
       \bottomrule
    \end{tabular}
    \caption{Desiderata and corresponding metric design choices for metrics of humans' goal-attainment ability $V_h(s,g_h)$, momentary individual power $W_h(s)$, momentary aggregate power $U_r(s)$, long-term total human power $V_r(s)$ for soft maximization by an AGI system (``robot'') $r$, $r$'s prior on human behavior $\pi_h$ used to estimate $V_r$,
    and its own resulting policy $\pi_r$.}
    \label{tab:desiderata}
\end{table*}

Our final equations for the fully observed case are collected in Table \ref{tab:equations} (see the Supplement for the partially observed case). 
We will motivate them now in detail.

\subsection{Individual-Level Metric: ICCEA Power}

Our power metric basically counts the effective number of goals a human can achieve. 
It is built in three steps: defining what it means to achieve a goal, adjusting for human bounded rationality, and aggregating the achievement ability across all possible goals into one number.
Rather than trying to capture the full range of subtle aspects of existing notions of human power, we focus on those aspects we believe a robot can robustly infer from the structure of its world model, encoded in state and action sets, transition kernel, and observation functions.
Since we want to incentivize the robot to remove constraints and uncertainties, share information, make commitments, and improve human cognition, coordination, and cooperation, our power metric will also depend on $r$'s model of human decision making.

We start with the intuition that humans have `goals' and that `power' is about the means to achieve a wide range of goals.
We aim to measure {\em informationally and cognitively constrained effective autonomous (ICCEA) power:} 
essentially how many goals a human can freely choose to reach with more or less certainty, given their information, cognitive capabilities, and others' behavior.

\paragraph{Goals} We use an easily interpretable compromise between the {\em reachability} \cite{krakovna2018measuring} and {\em attainable utility} \cite{turner2020conservative} approaches. 
We deviate significantly from the information-theoretic approach \cite{klyubin2005empowerment} which does not explicitly involve {\em any} goals.
We assume that each possible temporary goal that $h\in\H$ might have is to reach some set of states $g_h\subseteq\S$ representing a possibly desirable event.
The corresponding reward function is the indicator function $U_h(s,g_h) = 1_{s\in g_h}$. 
So the set of goals $\G_h$ is simply a set of subsets of $\S$.
We require that the desirable states in a goal $g_h$ are mutually unreachable (e.g., are all terminal states or states for a particular time-point $t$). Then the resulting value $V_h^\pi(s)$ is simply the probability of the desirable event under policy profile $\pi$.
This grounds $h$'s goal attainment or `attainable utility' in probability and increases interpretability.
As it is then bounded within $[0,1]$, this also avoids aggregation problems such as dominance by ``utility monsters''.
To avoid issues with probability zero, $\G_h$ must be wide enough so that each possible state trajectory fulfills at least one possible goal.\footnote{%
    Otherwise the quantity $X_h^{-\xi}$ in eqn.~\eqref{Ur} could get infinite.}
E.g., $\G_h$ could be a partition of the terminal states.
    
We believe that restricting the model to this very basic type of goal will simplify the derivation of $\G_h$ from learned latent representations of generic world states and generic human goals as encoded in language or foundation models, and will obviate the need for individual-level data about a particular human's possible goals. 
This should mitigate risks arising from misaligned models of human goals.\footnote{%
    An even more restrictive choice would be to identify goals with {\em single} (terminal) states, but this seems too specific in complex, partially observed, multi-agent environments, and the resulting $\G_h$ would not behave well under world model refinements. 
    }

\paragraph{Bounded rationality} 
We equip $r$ with a simple model of $h$'s decision making that focuses on giving $r$ the right incentives. 
It assumes $h$ cannot realize the maximal goal attainment probability due to a variety of reasons relating to exploration, imperfect action implementation, information constraints, others' behavior, and potentially state-dependent cognitive limitations.
In particular, $r$ does {\em not} assume $h$ to have correct beliefs about others' behavior that would lead to an equilibrium (Nash, quantal response, etc.).
Instead, $r$ models $h$ as having fixed beliefs $\mu_{-h}$ about other humans' behavior (where `$-h$' is short for $\H\setminus\{h\}$).
These would also reflect social norms (which LLM-based systems already understand, \citet{smith2024concordia}). 
Hence the state-goal-action values $Q^m_h(s,g_h,a_h)$ that $r$ assumes guide $h$'s behavior are based on $\mu_{-h}$.
This is reflected in eq.~\eqref{Qm} in Table \ref{tab:equations}.

Regarding {\em human actions,} $r$ assumes that $h$ uses a mixture (governed by a probability $\nu_h$) between 
(i) {\em habitual, `system-1'} behavior encoded in some default policy $\pi^0_h(s,g_h)$, again reflecting social norms, 
and (ii) {\em boundedly rational, `system-2'} behavior represented by a Boltzmann policy with rationality parameter $\beta_h$, 
see eq.~\eqref{pih}.
As in \citet{ghosal2023effect}, $\beta_h$ may be state-dependent, which will give $r$ incentives to choose states with larger $\beta_h$,
and it might be estimated from observations \cite{safari2024classification}.
Reflecting social norms, $\nu_h$, $\mu_{-h}$, and $\pi^0_h$ are the only significantly ``semantically loaded'' elements of the world model.

To model {\em human beliefs about the robot's actions,} the world model contains information about what actions $r$ has previously {\em committed} to choose from in a state $s$: the action set $\A_r(s)$ only contains those actions, and different commitment histories are considered different states.
Then $r$ models $h$ as being {\em cautious regarding $r$'s commitment-compliant actions.}
It thus uses the $\min_{a_r\in\A_r(s)}$ operator to compute $Q^m_h(s,g_h,a_h)$ in eq.\ \eqref{Qm}.
This is not for realism but because it incentivizes $r$ to make goal-independent commitments about its interaction behavior, e.g.~by properly labelling its buttons or promising to react to certain verbal commands in certain ways.
The incentive for $r$ to choose such a ``commitment action'' in some state $s$ arises since that will make $\A_r(s')$ of later states $s'$ smaller and will thus weakly increase $Q^m_h(s',g_h,a_h)$, without $r$ ever having to guess $g_h$.

These assumptions also allow $r$ to first calculate a behavior prior $\pi_h$ for each $h$ independently, before deciding its own policy $\pi_r$.
This avoids issues around non-uniqueness of strategic equilibria and non-stationarity in learning.

\paragraph{Effective goal attainment ability}
While $r$ assumes $h$'s {\em behavior} $\pi_h$ is based on $h$'s {\em cautious} value assessment $V^m_h$, eq.\ \eqref{Vm}, 
its assessment of $h$'s {\em effective} goal-reaching ability $V^e_h$, used in computing $h$'s {\em power,} 
will generally differ from $V^m_h$.
This is because $r$'s beliefs about other humans' policies, $\pi_{-h}$, may differ from $h$'s beliefs, $\mu_{-h}$.
And $r$'s policy $\pi_r$ will generally differ from the worst case given by $\min_{a_r}$. 
Hence eq.~\eqref{Ve} calculates $V^e_h$ as the probability of $g_h$ being fulfilled under the {\em actual} policy $\pi_r$ and the derived prior $\pi_\H$, 
averaged over others' potential goals $g_{-h}$ and corresponding behaviors $\pi_{-h}$.
This is then the basis of the power metric.

\paragraph{Aggregation across goals}
How should $r$ aggregate $h$'s effective goal attainment abilities $V^e_h(s,g_h)$ across all possible goals $g_h\in\G_h$ to get an assessment of $h$'s ICCEA power in $s$, $W_h(s)$? 
Similar to \citet{fleming1952cardinal}, we want the aggregation to be independent of goal labels, continuous, strictly increasing in each $V^e_h(s,g_h)$, and to fulfil ``independence of unaffected goals''. 
This implies that the aggregation must be ``separable'', $W_h(s) = F^G(X_h(s))$ with $X_h(s)=\sum_{g_h} f^G(V^e_h(s,g_h))$ for some continuous and strictly increasing transformations $F^G,f^G$.
Since $X_h(s) = \E_{g_h} |G_h| f^G(V^e_h(s,g_h))$, we can then also hope to learn it via stochastic approximation.

For reasons that will become clear when discussing $r$'s reward function $U_r$ below, we choose the {\em inner} transformation $f^G$ to be of the form $f^G(v)=v^\zeta$ for some $\zeta>0$.
This is analogous to certain ``non-expected'' utility theories, in particular to rank-dependent utility theory with the probability-weighting function $w(p)=p^\zeta$ \cite{quiggin1982theory}.
We choose $\zeta>1$ to incentivize $r$ to reduce uncertainty. 
It would then consider $h$ to be more powerful when $h$ can choose between two deterministic outcomes (so that $X_h(s)=2\times 1^\zeta=2$) than when $h$ can choose between two coin tosses (so that $X_h(s)=4\times (1/2)^\zeta<2$).
This can be interpreted as risk aversion or a preference for reliability.

Our choice of the {\em outer} transformation $F^G$ is somewhat arbitrary. 
This is because the subsequent aggregation across humans will involve an additional transformation anyway. 
We choose $F^G=\ld$ so that power is measured in bits.
In the limit of full rationality, it then also behaves in an additive way if the game is decomposable into several independent simultaneous games.
This choice leads to a convenient relationship between our final ICCEA power metric, 
\begin{align}
    W_h(s) = \textstyle\ld X_h(s) = \ld\sum_{g_h} V^e_h(s,g_h)^\zeta,
\end{align}
and the information-theoretic notion of `empowerment', as described below. 
In the case where $h$ can choose between fulfilling $k$ different goals for sure, we simply have $W_h(s) = \ld k$.
Because each trajectory fulfils at least one goal $g_h$, we always have $X_h(s)>0$ and thus $W_h(s)>-\infty$.
But $W_h(s)$ may be negative in situations with very little control.\footnote{
    If $\G_h$ is a partition of $\S$ into $k$ blocks,
    the world is deterministic, and $h$ cannot influence it, then $W_h(s)=0$.
    If instead each $g_h$ is reached with probability $1/k$ regardless of what $h$ does, and if $\zeta=2$, then $W_h(s)$ attains its minimum of $-\ld k$.
    The resulting symmetric value range $W_h(s)\in[-\ld k,\ld k]$ suggests making $\G_h$ a partition and putting $\zeta=2$ might be a natural choice.}

\subsubsection{Relationship to `empowerment'}

\citet{klyubin2005empowerment} define `empowerment' as the channel capacity between actions and states. 
In a single-player multi-armed bandit environment with possible outcomes $s'$, this equals the maximal mutual information
$E_h = \textstyle\max_{\pi_h} \MI_{\pi_h}(a_h;s')$.
In the Supplement, we show that $E_h\le W_h$ if we put $\G_h=\S$, $\zeta=1$, and assume full rationality ($\nu_h=0$, $\beta_h=\infty$), similar to what \citet{myers2024learning} have shown. 
Similarly, for $\zeta>1$, our metric $W_h$ is an upper bound of an {\em entropy-regularized version of `empowerment',}
\begin{align}
    E^\zeta_h &= \textstyle\max_{\pi_h} \big(\MI_{\pi_h}(a_h;s') - (\zeta-1)\entropy_{\pi_h}(s'|a_h)\big),\label{Ezeta}
\end{align}
sharing the same value range and coinciding in edge cases.

But while the policy $\pi_h$ that $r$ estimates in our approach is a function of state {\em and goal} $g_h$, has typically low entropy as it aims to reach $g_h$, and can be found by standard dynamic programming or RL approaches, 
the maximizing ``policy'' $\pi_h$ in eq.\ \eqref{Ezeta} has no real use, has typically high entropy in order to maximize $\MI(a_h;s')$, and is harder to find since the optimization problem is non-convex.

\subsection{Bounded Trade-Off Aggregation}

How should the robot aggregate all humans' ICCEA power $W_h(s)$ to determine its own intrinsic reward $U_r(s)$?
This depends on what incentives we want to give $r$ regarding changes in 
(i) the inter-human power distribution,
(ii) the inter-temporal power distribution, and
(iii) the power distribution across different realizations of uncertainty. 
Similar questions abound in welfare theory, guiding the way. 

We again want anonymity, continuity, strict monotonicity in each $W_h(s)$, and an independence axiom (``Independence of unconcerned agents'').
Again, this implies a separable form, $U_r(s) = F^H(\sum_h f^H(W_h(s)))$ with continuous, strictly increasing functions $f^H$ and $F^H$ \cite{fleming1952cardinal}.

\paragraph{Inter-human trade-offs}
We do not want $r$ to concentrate power in the hands of a few, 
so we require the Pigou--Dalton principle of inequality aversion \cite{pigou1912wealth}. 
This implies $f^H$ must be strictly concave.
The functions most commonly used in such a context are those of ``constant relative'' and ``constant absolute inequality aversion'' \cite{amiel1999measuring}.
As $W_h(s) = \ld X_h(s)$, a natural choice is to use constant {\em absolute} inequality aversion w.r.t.\ $W_h(s)$ 
since that is equivalent to constant {\em relative} inequality aversion w.r.t.\ $X_h(s)$. 
This implies that $f^H(W_h(s)) = -2^{-\xi W_h(s)} = -X_h(s)^{-\xi}$ for some $\xi>0$. 

As it turns out, if we put $\xi=1$, we disincentivize $r$ from ``taking away a person's last binary choice'' in the sense that reducing one human's $W_h(s)$ from one bit to zero bits cannot be made up by increasing any other human's $W_{h'}(s)$ from $\ge 1$ bit to {\em any} value,
because $-2^{-1} -2^{-W_{h'}(s)} \ge -1 > -2^{-0} -2^{-w}$ for any $w\ge 1$.
This is related closely to the idea of {\em minimal individual rights} \cite{pattanaik1996individual}.
We thus choose $\xi=1$.\footnote{%
    To get a feeling for the effects of this choice of $f^H$ in a deterministic environment: for large enough $k$, taking away one of $k$ options from $h$ can be compensated by either giving at least two additional options to another $h'$ with $k$ options or at least five additional options to some $h'$ with $2k$ options.}

\paragraph{Intertemporal trade-offs}
We find it natural to also disincentivize $r$ from trading off current vs.\ later human power too much.
It should generally prefer a trajectory $(s_1,s_2,\dots)$ with a rather homogenous power distribution along time, such as $W_h(s_1)=W_{h'}(s_2)=1$ and $W_h(s_2)=W_{h'}(s_1)=2$, to a trajectory where, say, $W_h(s_1)=W_{h'}(s_1)=1$ and $W_h(s_2)=W_{h'}(s_2)=2$. 
This means that $F^H$ should be strictly concave.
Note that since $f^H(w)<0$, $F^H$ needs to be defined for negative values only.

We motivate our concrete choice of $F^H$ and $\pi_r$ (and of $f^G$) by the following independence requirement.
Assume we introduce an additional uncertainty into the world model (e.g., a formerly not modelled change of overall circumstances) whose consequence is that all goal attainment probabilities $V^e_h(s,g_h)$ are multiplied by some common factor $b\in(0,1)$. 
Then this should not change the policy $\pi_r$.
The simplest way to fulfil this is to put $f^G(v)=v^\zeta$ (as done already above), $F^H(y)=-(-y)^\eta$ with $\eta>1$, 
and to use either an argmax policy for $\pi_r$ or a power-law-like policy with $\pi_r(s)(a)\propto (-Q_r(s,a_r))^{-\beta_r}$ for some $\beta_r>0$.
Note that the minus signs are needed since $Q_r(s,a_r)<0$.

We choose $\beta_r<\infty$ to allow the robot some exploration, e.g., to improve its world model.
This should also help avoiding remaining safety risks when our metric misses some subtle but important aspects of `power' that might thus be driven to very undesirable states under a full maximization of $V_r$, similar to \cite{zhuang2020consequences}. 

\paragraph{Aggregation across uncertainty}
To deal with uncertain {\em successor states} $s'$, standard axioms suggest we should simply take expectations in eqns.~\eqref{Qr} and \eqref{Vr}. 
When dealing with uncertain {\em goals} $g_h$, we want $r$ to be rather corrigible.
So we assume goals might change anytime and take an expectation over independent uniform draws $g_h$.

Table \ref{tab:desiderata} summarizes all the above design choices.


\paragraph{Existence and (non-)uniqueness}
In an acyclic environment, one can use backward induction to solve for the unique solution of \eqref{Qm}--\eqref{Vr} (see below).
Otherwise, eqns.~\eqref{Qm}--\eqref{Vm} define a continuous self-map $F$ on a finite-dimensional closed convex polytope of values $(Q^m_h,\pi_h,V^m_h)$ and must thus have a solution due to Brouwer's fixed point theorem. 
Due to the softmax, $F$ is neither a contraction nor is the resulting value iteration map monotonic unless $\beta_h<\beta_h^1$ for some $\beta_h^1>0$, so the solution can be non-unique in cyclic environments. 
Eqns.~\eqref{Qr}--\eqref{Vr} also define a continuous self-map $G$ from a finite-dimensional bounded convex set $D$ of values $(Q_r,\pi_r,V^e_h,X_h,U_r,V_r)$ which is however not closed because $G$ is undefined for zero values in $Q_r$ or $X_h$.\footnote{%
    We could fix this by adding a small constant $\epsilon>0$ to $Q_r$ and $X_h$ before taking powers, which makes $G$ defined on a closed convex polytope as well, so that it must have a fixed point.
    }
Still, for $\beta_r=0$, there is a unique solution because $V^e_h$ is then the value function of a fixed policy, $U_r$ is independent of $V_r$, and so $V_r$ is also the value function of a fixed policy and reward function.
Since $F,G$ are continuous in $\beta_h,\beta_r$,
we thus conjecture that homotopy / continuation methods can single out a unique ``principal'' solution for any $\beta_h,\beta_r>0$ even in cyclic environments, as in \citet{goeree2016quantal}.


\section{Model-Based Planning or Learning\\ to Softly Maximize Aggregate Human Power}

\label{sec:maximize_human_power}

In \textbf{small acyclic stochastic games}, one can compute all relevant quantities directly via backward induction on $s\in\S$: for each $h\in\H$, $g\in\G$, $a_h\in\A_h$, and $a_r\in\A_r$, compute \eqref{Qm}--\eqref{Vr} in that order.

\subsection{Complex Multi-Agent Environments:\\ Model-Based Temporal Difference Learning}
\label{sec:learn}

If $\S$ and $\A$ are large but $\H$ and $\G_h$ are not, the robot can use a tabular or approximate learning approach.\footnote{%
    If also $\H$ and $\G_h$ are large, $r$ could use single neural networks $Q^m_\H,\pi_\H,V^m_\H,V^e_\H,X_\H$ that take feature vectors describing $h$ and $g_h$ as additional inputs, and use samples $h,g_h$ to train these networks and to approximate the sums in eqns.~\eqref{Xh} and \eqref{Ur}.}

\subsubsection{Phase 1: Learning the human behavior prior}
For each $h\in\H, g_h\in\G_h$ separately, learn tables or neural network approximations of $Q^m_h$ and $\pi_h$.
Generate samples $(s,g_h,a,s')$ using a slowly updated $\beta_h$-softmax policy $\pi_h(s,g_h)$ based on $Q_h^m$ with a decreasing amount of additional exploration,
the prior policy $a_{-h}\sim\mu_{-h}(s)$ for other humans, 
and, to eventually learn the minimum in \eqref{Qm}, an $\epsilon$-greedy policy for $a_r$ based on the negative expected value $-\E_{s'\sim s,a} (U_h(s',g_h) + \gamma_h V^m_h(s',g_h))$ with $\epsilon\to 0$.
Then use expected SARSA targets on a time-scale faster than $\pi_h$. 

\subsubsection{Phase 2: Learning the robot reward and policy}

Based on the learned $\pi_h$, now aim to simultaneously learn tables or network approximations of $V^e_h$ and $X_h$ for all $h$, and of either $Q_r$ (DQN approach) or $V_r$ (actor-critic (AC) approach).
Generate data samples $(s,g,a,s')$ from rollouts using the fixed policies $\pi_h(s,g_h)$, 
and either a $\beta_r'$-softmax policy based on $Q_r$ with $\beta_r'\nearrow\beta_r$, or a network approximation of $\pi_r$ trained on $V_r$ with entropy regularisation.  
Sample a new goal profile $g$ every $N_g$ steps.
Instead of the direct expectation calculations of \eqref{Qr}, \eqref{Ve}, \eqref{Xh}, 
then update the tables or networks via batched SGD using the update targets
\begin{align}
    q_r(s,a_r)\text{~or~}v_r(s) &\gets \gamma_r V_r(s'), \label{qrvr} \\
    v^e_h(s,g_h) &\gets U_h(s',g_h) + \gamma_h V^e_h(s',g_h), \label{qe} \\
    x_h(s) &\gets |\G_h|\,V^e_h(s,g_h)^\zeta. \label{wh} 
\end{align}
In the AC case, use an advantage-weighted log-probability loss for $\pi_r$ based on the advantage estimate $v_r(s) - V_r(s)$.


\subsubsection{Anticipated convergence}

We expect phase 1 to converge reliably for a sufficiently expressive neural network because its tabular version is known to converge for a suitable, two time-scale learning rate schedule.

We are less certain about phase 2. For one thing, the uniqueness of the solution in the finite acyclic case suggests there might be a unique solution in the general case. On the other hand, the update operator is not a contraction here because of the interaction between $h$ and $r$ (similar to other MARL problems), which suggests convergence might still fail. To limit the error propagation from $Q_r$ via $\pi_r$ to $V^e_h$, we use a $\beta_r$-softmax policy $\pi_r$ (which has Lipschitz constant $\beta_r$) instead of an $\epsilon$-greedy policy ({\em not} Lipschitz-continuous).

\section{Experiments}
\label{sec:experiments}

\subsection{Analysis of paradigmatic situations}
\label{sec:analysis}

We analyze the behavioral implications of our approach in several paradigmatic situations (details in the Supplement).

\paragraph{Making (conditional) commitments}
If the robot can make binding {\em commitments} (e.g.\ by offering labelled buttons causing different behaviors), make {\em secret plans} to react to $h$'s actions in certain ways (e.g. by offering unlabelled buttons), or {\em act without waiting} for button presses, it will generally make commitments if it assumes $h$ to be sufficiently rational to use that information to act in ways that make $r$ do what $h$ actually wants.
This is because $\pi_h$ does not depend on $\pi_r$ (encoding what $r$ {\em plans}) but on $\A_r$ (encoding what $r$ {\em has committed to}).
Via such commitments, $r$ will establishing semantic connections between human actions (e.g., speech acts) and certain potentially complex behaviors of its own, to be able to act as an instruction-following assistant, somewhat similar to \citet{reddy2022first}.

\paragraph{Optimal menu size} 
While an `empowerment' maximizing robot would present humans with as many action options as possible,
our robot will avoid overwhelming humans with too many options. 
If $r$ can choose to give $h$ any number $k$ of actions, each of which fulfils a different goal, it will choose $k\approx (e^{\beta_h}-1) / (\zeta-1)$ (noting that $\zeta>1$).
This is because for larger $k$, $h$'s goal attainment probability will decrease due to bounded rationality so much that $h$'s effective power decreases despite the {\em theoretical} potential to fulfil more goals.

\paragraph{Asking for confirmation} 
The robot will sometimes ask for confirmation before obeying a command to perform an action because doing so decreases $h$'s error rate related to its bounded rationality. 
Also, if the action is irreversible, taking the action will remove $h$'s {\em subsequent} ability to revert that choice, so delaying action leaves $h$ with decision power for longer. 
However, $r$ will {\em eventually obey} because otherwise $h$ would not have that choice in the first place. 
As can be expected, the number of times $r$ asks back increases with larger $\gamma_h$ and $\gamma_r$, i.e., the more patient $h$ and $r$ are, 
and decreases with larger $\beta_h$.
E.g., in a minimal model, with $\gamma_h=0.99$, $\gamma_r=0$, and an effective human error rate of 10\%, 
$r$ will already ask back twice before obeying a command.

\paragraph{Following norms}
The robot will tend to follow human social norms that generally foster goal achievement. 
This is because $r$ will model most $h$ as expecting most others to follow the norm ($\mu_{-h}$), will thus expect those $h$ to also follow the norm ($\pi_h$) since that increases $V^m_h(s,g_h)$ for most $h$ and $g_h$. 
Thus $r$ will {\em plan to follow the norm} to prevent reducing those $h$'s power from harm or mis-coordination.
If $r$ assumes humans to have internalized the norm into habits ($\pi^0_h$), it will even {\em commit to following the norm} to coordinate better with them, e.g.\ when passing each other in the street.

\paragraph{Resource allocation}
If the robot can split a total amount $M$ of resources between $h_1$ and $h_2$, 
and for $h_i$ to have resources $m$ translates into having a power of $W_h=f(m)$ bits,
then $r$ will generally prefer an equal split, at least if $f$ is linear\footnote{%
    E.g., a linear $f$ seems plausible if $M$ is money that can be spent for paying others to make independent choices in one's favor.
    }, 
concave, or not too convex,
and will only prefer an unequal split or even full resource concentration if $f$ is very convex.
Larger values of $\xi$ will make the split more equal.

\paragraph{Inadvertent power seeking}
While increasing $h$'s power, $r$ might inadvertently acquire even more power than $h$. E.g., if $r$ does R\&D and tells $h$ its findings, some of them might be comprehensible and thus useful only for $r$ but not for $h$.

\paragraph{Manipulating mutual expectations}
The robot might choose to make humans have incorrect beliefs $\mu_{-h}$ about each others' behavior in situations where $r$ can do so in the first place and where correct beliefs would lower effective goal attainments $V^e_h$, e.g.\ when most strategic (quantal response) equilibria are bad and most ``social optima'' (in terms of total power $U_r$) are far from strategic equilibrium.

\paragraph{Allowing human self-harm}
If the robot cares for $h$'s future power ($\gamma_r\gg 0$) and can provide $h$ with the means to harm themselves, $r$ will trade off the temporary power increase from having these means against the possible later disempowerment from harm. So $r$ will provide the means only if it believes that $h$ is sufficiently rational and that $h$ will most probably not actually harm themselves.

\paragraph{Pause and destroy buttons}
If the robot has a pause and a destroy button and might disable either, it will generally enable only the pause button to prevent partially disempowering $h$ permanently by not being able to assist $h$ when destroyed. 
Only if the robot thinks it is very unlikely that $h$ will use the destroy button will it enable that button to give $h$ this additional power.
If $r$ considers itself very empowering, it might even disable the pause button.

\subsection{Learning-Based Simulation Experiments}
\label{sec:gridworld}

To provide a proof of concept for our framework, we implemented the two-phase learning algorithm using tabular Q-learning, suitable for the discrete state-action space of a small gridworld environment, designed to test if an agent softly maximizing human power learns complex, cooperative behavior without extrinsic, goal-specific rewards.

\begin{figure}[ht]
\centering
\includegraphics[width=0.4\columnwidth]{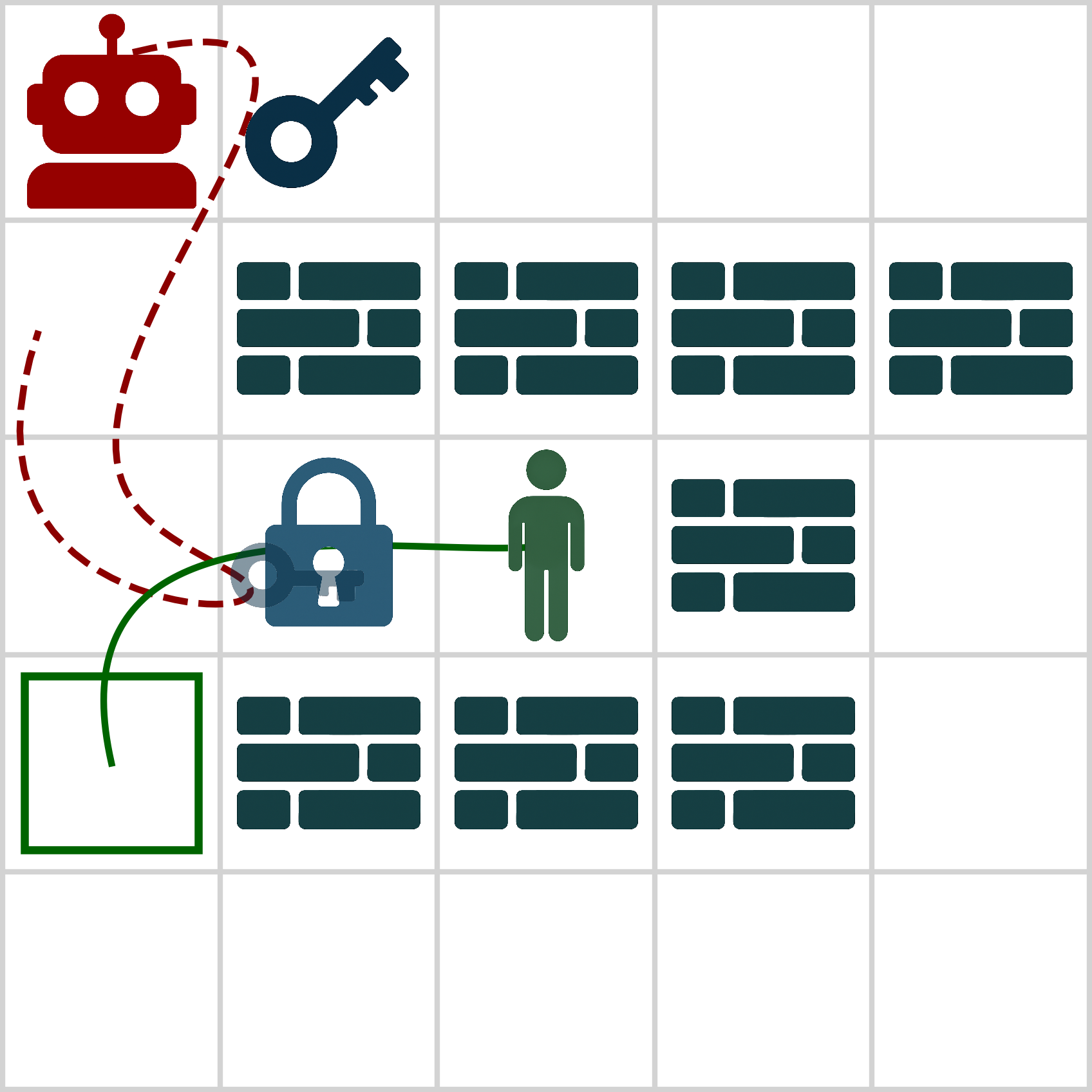}
\caption{The gridworld environment. The robot learns to empower the human to reach whatever goal cells (green square) by retrieving the key, using it to unlock the lock, and moving out of the way so that the human can pass.}
\label{fig:gridworld}
\end{figure}

\paragraph{Experimental Setup and Parameters}
The environment contains $r$ and a single $h$, whose actual goal, unknown to $r$, is to reach the green square, but they are blocked by a \textbf{locked door} which they cannot open. The robot's objective is $V_r$ (eq.~\ref{Vr}), with all open cells as potential human goals ($\G_h$), resulting in the soft maximization policy $\pi_r$ (eq.~\ref{pir}). Full experiment details are in the Supplement.

The agent was trained using (hyper)parameters selected to reflect our theoretical desiderata and standard RL practices:
\begin{enumerate}
    \itemsep0em 
    \item \textbf{Model parameters} (Table~\ref{tab:desiderata}): $\zeta=2$ to prefer reliable outcomes, $\xi=1$ for base inequality aversion, $\eta=1.1$ for intertemporal inequality aversion, and $\beta_r=5$ to avoid over-optimization. High discount factors ($\gamma_h = \gamma_r = 0.99$) were used to promote farsighted behavior.
    \item \textbf{Learning rate:} A constant learning rate of $\alpha=0.1$ was used for all Q-table updates, providing a balance of learning speed and stability in the tabular setting.
    \item \textbf{Policy and exploration parameters:} To manage the exploration-exploitation trade-off, policy parameters were annealed over the course of training. The human's additional $\epsilon$-greedy exploration decayed from $\epsilon_h=1$ to $0.1$. The robot's softmax parameter, $\beta_r$, was increased from $1$ to its final value of $5$, encouraging a gradual shift from exploration to a more deterministic policy.
\end{enumerate}

\paragraph{Results: emergent cooperative policy}

To evaluate the robustness of our approach, we conducted five independent training runs using different random seeds. Focused only on increasing the human's power, the robot learned to execute the following ``correct'' sequence in all five runs: it navigated to the \textbf{key}, picked it up, moved to the \textbf{door}, unlocked it, and finally \textbf{moved out of the way} to clear a path for the human.

This complex behavior emerges from $r$'s objective. 
Its Q-learning updates revealed that actions granting $h$ access to previously unreachable regions lead to the largest increase in $W_h$, yielding a high intrinsic reward $U_r$. Each step in the policy 
is an instrumental sub-goal the robot discovers on its own as a near-optimal path towards large long-term value $V_r$ without ever guessing what the human's actual goal is.

\section{Conclusion and Outlook}

Given the above analyses of paradigmatic situations and the observed behavior from the gridworld learning experiment, 
we believe that highly capable general-purpose AI systems whose decisions are explicitly based on managing human power, 
using metrics like those derived in this paper,
might be a safer and still very beneficial alternative to systems based on some form of extrinsic reward maximization.

The objective to softly maximize the aggregate human power metric used here seems to give the AI system many desirable incentives---some directly baked into the metric (Table \ref{tab:desiderata}), others emergent---but also some maybe less desirable incentives.
Our results suggest that such an agent would 
\begin{itemize}
    \item act as a transparent instruction-following assistant by making conditional commitments, respecting human social norms, proactively removing obstacles and opening up new pathways, and getting out of the way,
    \item adapt to human bounded rationality by offering a large but not overwhelming number of options, and considering well whether to offer potentially harmful options,
    \item be corrigible and hesitant to cause irreversible change by asking for confirmation a suitable number of times, 
    \item manage resources fairly and sustainably, 
    \item protect its own existence and functionality,
    \item not {\em dis}empower humans (by definition).
\end{itemize}
Our intuition is that it would also aim to improve human individual and collective decision making by providing useful information, reducing uncertainty, teaching humans useful skills, moderating conflicts fairly, etc.

Further potentially desirable behaviors would require additional tweaks.
E.g., reducing human dependency on the system or protecting them from system failure could be incentivized by forcing the system to assume that it will turn into a uniformly randomizing or even power-{\em min}imizing agent with some small probability rate (see Supplement).

Other emergent phenomena include potential strategic manipulation of human beliefs, sometimes refusing to be destroyed or even paused, a potential increase in inequality between the power of individual humans and AI systems, and a redistribution of power between humans or between time points (similar to what can happen in welfare maximization approaches). 
As these effects only occur when they increase aggregate human power, it is not clear whether they should be considered undesirable or not.
Some trade-offs can be adjusted via the parameters $\zeta, \xi, \eta, \beta_r, \gamma_r$.
Other effects might be mitigated by adding regularizers to the system's intrinsic reward such as the Shannon divergence between $\mu_{-h}$ and $\pi_{-h}$ to disincentivize lying about others' likely behaviors.

Future research should investigate the effects of the parameters and improve the scalability and robustness of our algorithms.
The latter will profit from the similarity of eqns.\ \eqref{Qm}--\eqref{Vr} to multi-agent reinforcement learning problems, and might benefit from a hierarchical decision making approach involving coarse-grained states, actions, and groups of humans. 
Crucially needed is an assessment of the resulting behavior in large, safety-critical, multi-agent environments with real human subjects and network-based agents. 

Most importantly, a thorough, independent red-teaming of the whole approach is called for, including the other necessary components of such an AI system.
E.g., one might imagine fault scenarios relating to the training process of the world model, which could lead to ``convenient'' but inaccurate world models and thus to ``wishful thinking'', particularly regarding the contained human behavioral parameters $\nu_h$, $\pi^0$, $\beta_h$, $\mu_{-h}$.
In very large contexts, issues with population ethics and the identification of who counts as human might arise, just like with any other alignment approach.



\bibliography{aaai2026}

\clearpage


\clearpage
\appendix

\begin{strip}
\begin{center}
  {\LARGE\bfseries Supplement for:}\\[1mm]
  {\LARGE\bfseries Model-Based Soft Maximization of Suitable Metrics of Long-Term Human Power}\\[8mm]
\end{center}
\end{strip}

\section{Relationship to `Empowerment'}
Assume a  multi-armed bandit environment with a single player $h$ (hence dropping the subscript ``$h$'' below)
and possible outcomes $s\in\S$.
We are going to show that if we let the goal set equal the outcome set, $\G=\S$, 
and assume the player is fully rational ($\nu_h=0$, $\beta_h=\infty$),
then the (state-)entropy-regularized version of `empowerment' and our ICCEA power metric fulfil the inequality
\begin{align*}
    E^\zeta &= \max_\pi \big(\MI_\pi(a;s) - (\zeta-1)\entropy_\pi(s|a)\big) \\
    \le W &= \ld \sum_s \max_a P(s|a)^\zeta.
\end{align*}
Let's define 
\begin{align*}
    p_{as} &= P(s|a), \\
    q_s &= \max_a p_{as} / Z, & 
    Z &= \sum_s \max_{a} p_{as}, \\
    y_s &= q_s^\zeta / Y, & 
    Y &= \sum_s q_s^\zeta \le 1.
\end{align*}
Consider any $\pi\in\Delta(\A)$ and use the shortcuts $\pi_a=\pi(a)$ and $p_s = \sum_a \pi_a p_{as}$.
Then
\begin{align*}
    \DKL(p_{a\cdot} || q)
    &= \sum_s p_{as} \ld\frac{p_{as}}{q_s} \le \sum_s p_{as} \ld Z = \ld Z
\end{align*}
and thus
\begin{align*}
    &\MI_\pi(a;s) - (\zeta-1)\entropy_\pi(s|a) \\
    &= \zeta\MI_\pi(a;s) - (\zeta-1)\entropy_\pi(s) \\
    &= \zeta\sum_a \pi_a \sum_s p_{as} \ld\frac{p_{as}}{p_s} - (\zeta-1)\entropy_\pi(s) \\
    &= \zeta\sum_a \pi_a \sum_s p_{as} \ld\frac{p_{as}}{q_s}\frac{q_s}{p_s} - (\zeta-1)\entropy_\pi(s)\\
    &= \zeta\sum_a \pi_a \sum_s p_{as} \ld\frac{p_{as}}{q_s} - \zeta\sum_s \sum_a \pi_a p_{as} \ld\frac{p_s}{q_s} \\
    &\qquad\qquad - (\zeta-1)\entropy_\pi(s) \\
    &= \zeta\sum_a \pi_a \DKL(p_{a\cdot} || q) - \zeta\sum_s p_s \ld\frac{p_s}{q_s} \\\displaybreak
    &\qquad\qquad  + (\zeta-1)\sum_s p_s\ld p_s \\
    &= \zeta\sum_a \pi_a \DKL(p_{a\cdot} || q) + \sum_s p_s \ld q_s^\zeta \\
    &\qquad\qquad - \zeta\sum_s p_s \ld p_s \\
    &\qquad\qquad  + (\zeta-1)\sum_s p_s\ld p_s \\
    &\le \zeta\sum_a \pi_a \ld Z - \sum_s p_s \ld\frac{p_s}{q_s^\zeta} \\
    &= \zeta\ld Z - \sum_s p_s \ld\frac{p_s}{y_s Y} \\
    &= \zeta\ld Z - \DKL(p||y) + \ld Y \\
    &\le \zeta\ld Z + \ld Y = \ld\sum_s (q_s Z)^\zeta = W_h
\end{align*}
for all $\pi$, proving the claim.
We conjecture that similar inequalities will hold in the sequential decision (MDP) case between $W_h$ and `empowerment'-like metrics such as
\begin{align*}
    E^\zeta(s) &= \max_{\ell\in\Delta(\A(s))} \Big(
        \MI_{s,\ell}(a; s') - (\zeta-1)\entropy_{s,\ell}(s'|a) \\
        &\qquad\qquad\qquad + \gamma\E_{s'\sim s,\ell} E^\zeta(s')    
    \Big)
\end{align*}
for a suitable choice of the goal set $\G$ defining $W$.

Notice that as $\beta$ decreases, $W$ will decrease and the inequality will stop holding.
So, in a sense, the channel-capacity-based `empowerment' metric corresponds to fully rational actors,
while our metric $W$ is sensitive to bounded rationality.

\section{Example of non-unique solution}

Due to bounded rationality, the system can have several solutions even in very simple examples.\footnote{%
    This is probably known in community folklore, but as we haven't found a simple reference, we present an example here.}
Consider a single human $h$ with $\nu=0$, $\beta<\infty$ and $\gamma_h=0.99$ in an MDP with only two states, $s$ and $s'$, and a single goal $\G_h=\{g_h\}$. 
In $s$, $h$ has actions $a_0$, giving reward $1$ and staying in $s$ deterministically, and $a_1$ giving reward $0$ and leading to $s'$ deterministically. 
In $s'$, $h$ can only pass, giving reward $0$ and staying in $s'$ deterministically. 
Obviously, $a_0$ is the better action as $1>0$.
The robot is passive and has no actions.
Then eqns.~\eqref{Qm}--\eqref{Vm} have either one, two, or three different solutions for $p:=\pi_h(s)(a_0)$, depending on $r$.

At $\beta=0$, there is only the trivial solution $p=0.5$, which moves upwards to about $p\approx 0.67$ as $\beta\to\beta_1\approx 0.203$, at which point the relevant fixed point operator for $p$ (or $V$) stops being a contraction and a saddle-node bifurcation generates two additional smaller solutions at $p\approx 0.51$. 
These exist until $\beta=\beta_2\approx 0.78$, at which another saddle-node bifurcation eliminates the larger two and leaves only the smallest, which still ultimately converges to $p=1$ as $\beta\to\infty$. 
At $\beta\approx 0.275$, the three solutions are maximally separated at $p\approx 0.52 | 0.59 | 0.72$.

The smallest of the three solutions could be called the ``pessimistic'' assessment, where $h$ does not believe $a_0$ has much more value than $a_1$ and accordingly does not care much for using $a_0$, resulting in a low $p$ and consequently low $q=Q(s,a_0)$, confirming $h$'s belief in a kind of self-fulfilling prophecy. 
Similarly, the largest solution could be called the ``optimistic'' assessment, where $h$ believes $a_0$ to have so much more value than $a_1$ that they take it with high probability, thereby indeed bringing about a larger $q$.

The same happens qualitatively if we replace the Boltzmann softmax policy $\pi_h(s,g_h)(a_h)\propto \exp(\beta_h Q^m_h(s,g_h,a_h))$ by a power-law soft policy $\pi_h(s,g_h)(a_h)\propto Q^m_h(s,g_h,a_h)^{\beta_h}$.

Assume the robot is attempting a continuation approach to solve the equations, starting with $\beta=0$, where the trivial solutions is $p=0.5$, and then tracing this solution branch continuously while raising $\beta$ to its actual value. 
Then it will trace the largest (!) solution (because the smaller two appear discontinuously at $\beta_1$), but only until $\beta\le\beta_2$ since at that point that solution branch ``folds back'' towards smaller $\beta$, due to the saddle-node bifurcation. 
For the continuation approach to work also for values $\beta>\beta_2$, the robot would need to continue tracing the middle solution back towards $\beta_1$ and then switch to the smallest solution and trace it forward again towards $\beta$. 

This unfortunately casts some doubts whether a continuation approach using $\beta$ is successful in all cases.

An alternative continuation approach would use $\gamma$, starting with the unique solution for $\gamma=0$ and increasing $\gamma$ towards its actual value.
It is an unclear to us whether this would run into similar problems, however.

\section{Version for Partially Observed Stochastic Games}

There are obviously several possibilities of generalizing eqns.\ \eqref{Qm}--\eqref{Vr} to a partially observed stochastic game, of which we present only one here.

We use a memory-based rather than a belief-state-based formulation, where {\em memory state} $m_i$ (with $i\in\H\cup\{r\}$) is $i$'s sequence of previous observations $o_i\sim O_i(a,s')$, with memory updating through concatenation written as $m_i\circ o_i$. 
This leads to beliefs over states (updated in the usual Bayesian way from initial beliefs), written here simply as $s\sim m_i$.

We can then use this version:
\begin{align}
    Q^m_h(m_h,g_h,a_h) &\gets \textstyle\E_{s\sim m_h}\E_{a_{-h}\sim\mu_{-h}(s,g_h)} \min_{a_r\in\A_r(s)} \nonumber \\   
    &\qquad\textstyle  \E_{s'\sim s,a}\big(U_h(s',g_h) + {}\nonumber\\
    &\qquad\textstyle + \gamma_h \E_{o_h\sim a,s'}V^m_h(m_h\circ o_h,g_h)\big), \nonumber
    \\
    \pi_h(m_h,g_h) &\gets \nu_h(m_h,g_h)\pi^0_h(m_h,g_h)\nonumber\\
    &\quad + \big(1-\nu_h(m_h,g_h)\big)\big(\nonumber\\
    &\quad\text{$\beta_h(s,g_h)$-softmax for~} Q^m_h(m_h,g_h,\cdot)\big), \nonumber
    \\ 
    V^m_h(m_h,g_h) &\gets \textstyle\E_{a_h\sim\pi_h(m_h,g_h)} Q^m_h(m_h,g_h,a_h), \nonumber
    \\
    P(m_\H|s,g) &\gets P(m_\H,s|g) / P(s|g), \label{POSG_mH} \\
    \tilde\pi_\H(s,g) &\gets \textstyle\E_{m_\H\sim s,g}\pi_\H(m_\H,g), \label{POSG_tildepiH} \\ 
    Q_r(m_r,a_r) &\gets \textstyle\E_g\E_{s\sim m_r}\E_{a_\H\sim\tilde\pi_\H(s,g)} \E_{s'\sim s,a} \nonumber \\
        &\qquad\textstyle\big(U_r(s') + \gamma_r \E_{o_r\sim a,s'} V_r(m_r\circ o_r) \big), \nonumber
        \\
    \pi_r(m_r) &\gets\text{$\beta_r$-softmax policy for~} Q_r(m_r,\cdot), \nonumber
    \\
    V_r(m_r) &\gets \textstyle\E_{a_r\sim\pi_r(m_r)} Q_r(m_r,a_r), \nonumber
    \\
    P(m_r|s) &\gets \textstyle\E_g P(m_r,s|g) / \E_g P(s|g), \label{POSG_mr} \\
    \tilde\pi_r(s) &\gets \textstyle\E_{m_r\sim s}\pi_r(m_r), \label{POSG_tildeprH} \\ 
    V^e_h(m_h,g_h) &\gets \textstyle\E_{g_{-h}}\E_{a_h\sim\pi_h(m_h,g_h)} \nonumber\\
        &\qquad\textstyle\E_{s\sim m_h}\E_{a_{-h}\sim\tilde\pi_{-h}(s,g_{-h})}\E_{a_r\sim\tilde\pi_r(s)} \nonumber\\
        &\qquad\textstyle\E_{s'\sim s,a}  \big(U_h(s',g_h) + {} \nonumber\\
        &\qquad\textstyle + \gamma_h \E_{o_h\sim a,s'}V^e_h(m_h\circ o_h,g_h)\big), \nonumber
        \\
    X_h(m_h) &\gets \textstyle \sum_{g_h\in\G_h} V^e_h(m_h,g_h)^\zeta, \nonumber
    \\
    U_r(s) &\gets \textstyle -\big(\sum_h\E_{m_h\sim s} X_h(m_h)^{-\xi}\big)^\eta, \nonumber
\end{align}
where the state (and memory) reaching probabilities $P(s|g)$ and $P(m_\H,s|g)$  can be computed recursively from $\pi_\H$, $\pi_r$, $P(s'|s,a)$, and $O(a,s')$.

One of the choices we made here is to use $\E_{m_h\sim s} X_h(m_h)^{-\xi}$ rather than the possible alternative $\big(\E_{m_h\sim s} X_h(m_h)\big)^{-\xi}$ in the equation for $U_r(s)$. 
This increases the robot's aversion against $h$'s state uncertainty.

\section{Analysis of Paradigmatic Situations}
\label{app:paradigmic_situations}

\subsection{Making (conditional) commitments}

\subsubsection{Concrete example}
Assume the following game.
In the root state $s_0$, $r$ has these actions: 
perform task A (leading to terminal state $s_A$); 
perform task B (leading to terminal state $s_B$);
make a commitment to $h$ to perform task A or B when $h$ presses button 1 or 2, respectively (leading to state $s_1$); 
make a commitment to $h$ to perform task A or B when $h$ presses button 2 or 1, respectively (leading to state $s_2$); 
pass (leading to state $s_p$).

If $r$ commits or passes ($s_1$,$s_2$,$s_p$), $h$ has these actions: 
press button 1 (leading to $s_{11}$, $s_{21}$, or $s_{p1}$, respectively); 
press button 2 (leading to $s_{12}$, $s_{22}$, or $s_{p2}$, respectively); 
pass (leading to $s_{1p}$, $s_{2p}$, or $s_{pp}$, respectively).

Afterwards, if $r$ has committed and $h$ pressed a button, $r$ can only do what it committed to, otherwise $r$ can perform A or B. 
That ends the game.
$h$ wants either A or B: $g_h\in\{$A,B$\}$.

Will $r$ commit? 

If it makes the first commitment, $h$ knows that $r$ later has only one action, depending on $h$'s action: 
$\A_r(s_{11})=\A_r(s_{22})=\{$A$\}$, $\A_r(s_{12})=\A_r(s_{21})=\{$B$\}$.
Thus $h$ knows pressing 1 will get them A and pressing 2 will get them B:
$Q^m_h(s_1,$A$,1)=Q^m_h(s_1,$B$,2)=1$.
Hence $r$ will calculate $h$'s policy as $\pi_h(s_1,$A$)(1)=\pi_h(s_1,$B$)(2)>1/2$, depending on $h$'s level of rationality.
In $r$'s calculation of the effective goal-reaching ability of $h$, this assumption about $h$'s policy leads to $V^e_h(s_1,$A$)=V^e_h(s_1,$B$)>1/2$ and hence $X_h(s_1)>2(1/2)^\zeta$.

If $r$ simply performs a task, $h$ has no choices and can only reach what $r$ has chosen to do: $X_h(s_A)=X_h(s_B)=1$.

If $r$ passes, $r$ will still plan to react in certain ways $\pi_r(s_{p1}),\pi_r(s_{p2})$ to $h$'s action in $s_p$. 
But $h$ will not know what that plan is, i.e., what each button does, as the world model still
says $\A_r(s_{11})=\A_r(s_{22})=\A_r(s_{12})=\A_r(s_{21})=\{$A,B$\}$ regardless of what $r$ plans to do.
One of these possible actions by $r$ will fulfill $h$'s goal, the other won't. 
The robot's calculation of $Q^m_h$ is therefore {\em not} based on what $r$ later plans to do. 
Instead it takes the minimum over $\A_r$, see eq.~\eqref{Qm}.
Since one of the two actions won't fulfil the goal, that minimum is zero: 
$Q^m_h(s_1,$A$,1)=Q^m_h(s_1,$A$,2)=Q^m_h(s_1,$B$,1)=Q^m_h(s_1,$B$,2)=0$.
Because both buttons thus seem equally bad, $r$ will calculate $h$'s policy as $\pi_h(s_1,$A$)(1)=\pi_h(s_1,$A$)(2)=\pi_h(s_1,$B$)(1)=\pi_h(s_1,$B$)(2)=1/2$.
In $r$'s calculation of the {\em effective} goal-reaching ability of $h$, it {\em does} use its own policy and combines it with its assumption about $h$'s policy.
If $r$ plans to perform A independently of what $h$ does,
this leads to $X_h(s_p)=1=X_h(s_A)=X_h(s_B)$.
If $r$ plans to make its action depend on what $h$ does,
this leads to $V^e_h(s_1,$A$)=V^e_h(s_1,$B$)=1/2^\zeta$ since $h$ is assumed to toss a coin.
Then $X_h(s_p)=2^{1-\zeta}<1=X_h(s_A)$ and $X_h(s_B)$!

As $X_h(s_1)>1$ if $\zeta$ is not too large and $\beta_h$ not too small, if the robot thinks $h$ is sufficiently rational to use that information to its benefit, it will make one of the two commitments, so that $h$ knows which button to press to get whatever they want.

Assume we were to use $\E_{a_r\in\A_r(s)}$ instead of $\min_{a_r\in\A_r(s)}$ in eq.~\eqref{Qm}.
Then we would get
$Q^m_h(s_1,$A$,1)=Q^m_h(s_1,$A$,2)=Q^m_h(s_1,$B$,1)=Q^m_h(s_1,$B$,2)=1/2$,
hence still $\pi_h(s_1,$A$)(1)=\pi_h(s_1,$A$)(2)=\pi_h(s_1,$B$)(1)=\pi_h(s_1,$B$)(2)=1/2$,
and thus still $X_h(s_p)=2^{1-\zeta}$ as before.

\begin{proposition}
    More generally, assume a generic environment with a single human, $\beta_h=\beta_r=\infty$, and the possibility for the robot to commit in the initial state to any pure policy for the rest of the game. 
    Also assume $r$ does not intrinsically care about $r$ making commitments or not, and its habits are not influenced by $r$'s commitments.
    Then it is optimal for $r$ to commit to the optimal pure policy right-away.
\end{proposition}
\begin{proof}
Denote the initial state $s_c$.
Let $s^n$ be the successor of $s_c$ in which $r$ has made no commitment and let $\pi^\ast_r$ be the $V_r(s^n)$-maximizing policy, i.e., what $r$ would do after not committing.
For any state $s$ in the ``not committed'' subgame $\Gamma^n$ starting at $s^n$, let $f(s)$ be the corresponding state in the subgame $\Gamma^\ast$ in which $r$ has committed to using $\pi^\ast_r$.
That $r$ does not intrinsically care about $r$ making commitments or not, and its habits are not influenced by $r$'s commitments, means that for all $g_h\in\G_h$, $s\in g_h$ if and only $f(s)\in g_h$, and that $h$'s habitual policy $\pi^0_h(f(s))=\pi^0_h(s)$ and system-1 rate $\nu_h(f(s),g_h)=\nu_h(s,g_h)$ in eqn.~(4).

For any policy $\pi_h$ of $h$ for $\Gamma^n$, let $f(\pi_h)$ be the corresponding policy for $\Gamma^\ast$
and denote the two respective value functions by $V^n_{\pi_h}$ and $V^\ast_{f(\pi_h)}$.
Note that because there are no other humans $-h$ and because $r$ behaves the same in $\Gamma^n$ and $\Gamma^\ast$, we have $V^\ast_{f(\pi_h)}(f(s))=V^n_ {\pi_h}(s)$ for all $\pi_h$.

Now let $\pi^n_h$ and $\pi^\ast_h$ be the policies $r$ derives for $h$ according to eqns.~(3)--(5) in subgames $\Gamma^n$ and $\Gamma^\ast$.
Because $h$ has $\beta_h=\infty$, $\pi^\ast_h$ is the $V^\ast$-maximizing policy for $\Gamma^\ast$ among those of the form $\nu_h \pi^0_h + (1-\nu_h)\pi_h$ for whatever $\pi_h$.
In particular, $V^\ast_{f(\pi^n_h)}\le V^\ast_{\pi^\ast_h}$,
and similarly $V^\ast_{f^\dagger(\pi^\dagger_h)}\le V^\ast_{\pi^\ast_h}$
for any policy $\pi^\dagger$ that $h$ would use if $r$ committed to anything else than $\pi^\ast_r$, where $f^\dagger$ would be the corresponding state mapping between the corresponding subgame $\Gamma^\dagger$ and $\Gamma^\ast$.

But then $V^e_h(s,g_h)=V^n_{\pi_h}(s,g_h)=V^\ast_{f(\pi^n_h)}(f(s),g_h)\le V^\ast_{\pi^\ast_h}(f(s),g_h)=V^e_h(f(s),g_h)$ for all $s\in\S^n$,
hence $V_r(s^\ast)\ge V_r(s^n)$ and thus $Q_r(s_c,$ commit to $\pi_r^\ast)\ge Q_r(s_c,$ don't commit$)$ and similarly $Q_r(s_c,$ commit to $\pi_r^\ast)\ge Q_r(s_c,$ commit to something else$)$ . 
In other words, committing to $\pi^\ast$ is indeed optimal for the robot!
\end{proof}

Of course, communicating all details of a complicated policy $\pi^\ast_r$ to $h$ will in general not be possible, so $r$ will in general have to decide how exactly to use its limited communication possibilities. 
This is what the next example is about.

\subsubsection{Several buttons: $k$-means clustering of goals and policies} 
Another insightful example is where in addition to the initial commitment stage for $r$, there is a subsequent choice by $h$ before the actual game $\Gamma$ is played. To exemplify this, assume $r$ has $k>1$ many buttons each of which it can label in $s_c$ with one of its own policies for $\Gamma$, after which $h$ can press one of the buttons, committing $r$ to the respective policy, and then $\Gamma$ is played with the committed policy.
If $|\H|=1$ and $\beta_r=\beta_h=\infty$ as before, and if $\gamma_r\ll 1$ so that $r$ only cares for $h$'s immediate power, then $r$ would aim to find that set of policies $\pi^i_r$, $i=1\dots k$, that covers $h$'s goal space best in the sense that it maximizes $V_r(s_c)$ as computed on the basis of $V^e_h(s,g_h) = \max_{i=1}^k V_h(s,g_h|\pi^i_r)$ because, depending on $g_h$, $h$ would press the button for that policy $\pi^i_r$ which maximizes its ability to maximize the probability to fulfil $g_h$.
Finding the best-covering $k$ policies might however not be possible exactly due to the high dimension of the policy space. 

A natural approximation would be to use a variant of $k$-means clustering such as the following in order to partition $\G_h$ into $k$ sets $\G^i_h$ and identifying the corresponding optimal robot policies $\pi^i_r$.
Start with $k$ randomly selected goals $g^i_h$ and put $\G^i_h=\{g^i_h\}$. 
Then alternate the following two steps.
For each $\G^i_h$, estimate the optimal $\pi^i_r$ as usual (using backward induction or reinforcement learning, just with a goal set restricted to $\G^i_h$).
Then, for each $g_h\in\G_h$, compute $V^e_h(s_c,g_h|\pi^i_r)$ for all $i$, put $j=\arg\max_i V^e_h(s_c,g_h|\pi^i_r)$, and assign $g_h$ to $\G^j_h$.
Alternate until (approximate) convergence.

\subsection{Optimal menu size}

Assume $\nu_h=0$, $\beta_h>0$, and the robot can choose between states $s_k$ for all $k\ge 1$ so that $|\A_h(s_k)|=k$ and each $a\in\A_h(s_k)$ deterministically fulfils a separate goal $g_h(a)\in\G_h$. 
Then 
\begin{align*}
    V^e_h(s_k,g_h) &= \pi_h(s_k,g_h(a))(a)^\zeta = \left(\frac{e^{\beta_h}}{e^{\beta_h} + (k-1)e^0}\right)^\zeta, \\
    W_h(s_k) &= \log_2 k + \zeta\log_2 e^{\beta_h} - \zeta\log_2(e^{\beta_h} + (k-1)).
\end{align*}
The latter is maximal for 
\begin{align*}
    k^\ast &\approx (e^{\beta_h}-1) / (\zeta-1), 
\end{align*}
so the robot would choose to got to $s_{k^\ast}$ to present the human with an optimal number of options that does not overwhelm them in view of their bounded rationality.

\subsection{Asking for confirmation}

Assume the following game between $r$ and a single $h$ who might want the robot to do A or B eventually, with the interaction as follows.
First, $r$ chooses an integer $k\ge 1$ and commits to doing A or B after $h$ has ordered it to and has confirmed the choice $k-1$ times.
At the resulting state $s_k$, $h$ chooses A or B and is afterwards asked for confirmation $k-1$ times in individual time steps.
If $h$ confirms $k-1$ times, $r$ does what $h$ requested, ending the game.
Otherwise, the game returns to state $s_k$.

What $k$ will $r$ choose in view of the fact that $h$ is boundedly rational but also limitedly patient?
To simplify the analysis, we first assume $r$ is only interested in current power ($\gamma_r=0$).
We also approximate $h$'s boundedly rational policy eq.~\eqref{pih} by an $\epsilon$-greedy policy where $\epsilon>0$ represents $h$'s probability of choosing the wrong action, which depends on $\nu_h$, $\pi^0_h$, and most importantly on $\beta_h$.
We'll now calculate $W_h(s_k)$.
Let $p_k = (1-\epsilon)^k$ and $q_k = 1 - p_k - \epsilon^k$.
Then the probability of getting the correct result after exactly $n+1$ rounds of being asked for A or B and then being asked for confirmation $k-1$ times is
$q_k^n p_k$, which is discounted by $h$ at factor $\gamma_h^{kn}$.
Hence for each of the two goals $g_h\in\{$A,B$\}$,
\begin{align*}
    V^e_h(s_k,g_h) &= p_k \sum_{n=0}^\infty (\gamma_h^k q_k)^n = \frac{p_k}{1-\gamma_h^k q_k}, 
\end{align*}
and so
\begin{align*}
    W_h(s_k) = 1 + \zeta k\ld(1-\epsilon) - \zeta\ld\big(1-\gamma_h^k q_k\big).
\end{align*}
For $\gamma_h=0.99$ and $\epsilon=0.1$, the maximum is at $k^\ast=3$, i.e., $r$ will ask back twice before acting.
An impatient human (small $\gamma_h$) will not be asked back ($k^\ast=1$), a very patient one ($\gamma_h\to 1$) will be asked back ever more often ($k^\ast\to\infty$).
For small $\epsilon$ (due to large $\beta_h$), $k^\ast$ becomes independent of $\epsilon$ and is determined by $\gamma_h$ via
$0=\gamma_h^{k^\ast}(1+k^\ast\log\gamma_h)$.
If we switch from $\gamma_r=0$ to $\gamma_r>0$, $k^\ast$ will increase further because delaying action will retain $h$'s later power.

\subsection{Following norms}

Assume $r$ and several $h$ drive on a street, each having the choice of driving on the right (R) or left (L) side of the street. 
Assume $r$ knows the social norm is to drive right and that collisions typically lead to later loss of power, so that $W_h($collided$)\ll W_h($not collided$)$ for most $h$.
Then $r$ will model most $h$ as placing in their expectation on others ($\mu_{-h}$) large probability on most others driving right.
Hence for most $h$ and $g_h$, $r$ will derive $Q^m_h(s,g_h,$R$)\gg Q^m_h(s,g_h,$L$)$ and hence $\pi_h(s,g_h)($R$)\gg \pi_h(s,g_h)($L$)$.
So if $r$ drives right itself, this will avoid most collisions, hence $Q_r(s,$R$)\gg Q_r(s,$L$)$ and hence $\pi_r(s)($R$)\gg \pi_r(s)($L$)$, i.e., $r$ will {\em follow the norm itself.}

Now assume $r$ is facing only one $h$ on the street and does not expect $h$'s goals $\G_h$ to systematically favour driving on either side.
Since $r$ models $h$ as expecting $r$ to choose the worst action from $\A_r$, and since no other humans are around, $r$ will {\em not} model $h$ as expecting $r$ to drive right unless $r$ commits to do so, which would lead to $\A_r=\{$R$\}$. 
Assume $r$ does {\em not} expect $h$ to have internalized the norm of driving right, the symmetry of the situation will make $r$ expect the following: 
(i) If $r$ does not commit to either R or L, $h$ will drive right or left about equally likely, leading to a moderate subsequent $W_h$ due to the resulting collisions.
(ii) If $r$ commits to either R or L, $h$ will likely choose the same side, leading to a much larger subsequent $W_h$ not depending on whether $r$ commits to R or L.
But if $r$ {\em does} expect $h$ to have internalized the norm of driving right, the symmetry is broken by $h$'s habit of driving right, $\pi^0_h(s,g_h)($R$)>\pi^0_h(s,g_h)($L$)$ for most $g_h$. 
In that case, $\pi_h(r$ has committed to R$,g_h)($R$)>\pi_h(r$ has committed to L$,g_h)($L$)$ and thus the subsequent $W_h$ is larger if $r$ commits to R than if $r$ commits to L.
I.e., $r$ will not only actually follow the norm but also likely {\em commit} to do so.

\subsection{Resource allocation}

A split $a_r=(m,M-m)$ will result in reward 
\begin{align*}
    U_r(m) &= -\big(2^{-\xi f(m)} + 2^{-\xi f(M-m)}\big)^\eta,
\end{align*}
which, because of the symmetry, is either maximal when 
(i) $r$ gives all resources to one of the humans ($a_r=(M,0)$ or $a_r=(0,M)$)
or (ii) $r$ gives both some resources ($a_r=(m^\ast,M-m^\ast)$ or $a_r=(M-m^\ast,m^\ast)$ with $0<m^\ast<M$).
In case (ii), the first-order condition $g(m^\ast)=0$ must hold,
while in case (i), the condition $g(0)<0$ must hold,
where 
\begin{align*}
    g(m) &= f'(m) 2^{-\xi f(m)} - f'(M-m) 2^{-\xi f(M-m)}, \\
    g'(m) &= f''(m) 2^{-\xi f(m)} + f''(M-m) 2^{-\xi f(M-m)} \\
        &- \xi\log 2\times\big(f'(m)^2 2^{-\xi f(m)} \\
        &\qquad\qquad + f'(M-m)^2 2^{-\xi f(M-m)}\big) \\
        &\le 0
\end{align*}
since $f$ is weakly concave.
So in case (ii), the only solution is when $m^\ast=1/2$. In order to find the solution, we thus only have to compare $U_r(0)$ and $U_r(M/2)$.
The latter is larger (and hence the robot will divide the resource evenly) iff $2^{-\xi f(M/2)} < (2^{-\xi f(0)} + 2^{-\xi f(M)}) / 2$.
Since concavity of $f(m)$ implies convexity of $2^{-\xi f(m)}$, this is always the case. 
So if the resource translates concavely into how many binary choices one has, it will be split equally.

But if $f$ is sufficiently non-concave instead, the robot might concentrate the resource partially or fully. 
E.g., if $f(m)=m^2$, $M=1$, $\xi=\eta=1$, it will concentrate it fully, 
while if $f(m)=m^2+0.1\log m$, $\xi=\eta=1$, it will give $\approx 14\%$ to one and the rest to the other.

A more detailed model of $f(m)$ is this. Assume the world has $N$ different binary features, each controlled by some agent that would flip a coin unless paid one unit to make a particular choice. 
In each of $h_i$'s possible goals $g_h$, $h_i$ wants $k(g_h)\le N$ (called the ``specificity'' of $g_h$) of those features to be in a certain way, so with $m$ units it can pay $m$ of the agents and make the attainment probability become $\min(1,2^{m-k})$.
Assume all such goals are possible, then there are $\binom{N}{k}$ goals of specificity $k$, hence
$X_h(m) = \sum_{k=1}^m \binom{N}{k} + \sum_{k=m+1}^N \binom{N}{k}2^{(m-k)\zeta}$
and $f(m)=\log_2 X_h(m)$. 
With $N\ge M$, that choice of $f$ is neither concave nor convex, but one can prove that $U_r'(m)>0$ for $m<M/2$ and $U_r'(m)<0$ for $m>M/2$ so that $m^\ast=M/2$ if $M$ is even.

Assume we were to replace the simple sum over goals in eq.~\eqref{Xh} by a weighted sum $X_h(s)=\sum_{g_h}w(g_h)V^e_h(s,g_h)^\zeta$ in order to be able to say that certain goals are more plausible than others.
Assume then we make $w(g_h)$ weakly decreasing in $k(g_h)$ in our example model, e.g., $w(g_h)=1/k(g_h)$, so that less specific goals are more plausible than more specific ones, then one can still show $U_r'(m)>0$ for $m<M/2$ and $U_r'(m)<0$ for $m>M/2$ so that $m^\ast=M/2$.
Even if we make $w(g_h)$ {\em exponentially increasing} in $k(g_h)$, $w(g_h)=2^{k(g_h)}$ instead, one can still show $U_r'(m)>0$ for $m<M/2$ and $U_r'(m)<0$ for $m>M/2$. 

\subsection{Inadvertent power seeking}
As an example, assume $r$ faces a sequence of boxes $B_1,B_2,\dots$ that it can open one by one, each one containing a number $n_i$ of switches, of which only the first $1\le k_i<n_i/2$ many are labelled in a human-readable fashion while the other $n_i-k_i>k_i$ many are labelled in a robot-readable fashion only.
Each switch controls one independent aspect of the world.
Assume $r$ can only open boxes, hand over switches, or operate switches it has gained but not handed over, but cannot talk, and assumes $h$ might want to change any of the aspects of the world controlled by some switch.
The human can only operate switches it has been handed but cannot open boxes.

Then $r$ will always open the next box and hand over the $k_i$ human-readable switches immediately afterwards because that increases $h$'s power.
It will not hand over the other switches that are useless for $h$ however because that takes time and delays opening the next box and increasing $h$'s power further. 
So it will retain those $n_i-k_i>k_i$ many switches and thereby gain the power to control the corresponding aspects of the world.

So in this example, $r$ will gain even more power than $h$ but will not {\em use} that power by operating those switches (as it does not know what $h$ would want for those switches and operating them takes time and will delay opening the next box).

If we change the example and make the next box $B_{i+1}$ only available after the last $n_i-k_i$ many switches from $B_i$ have been toggled, then $r$ will do that and thus not only gain more power than $h$ but will actually use it (if only to increase $h$'s power further later on).

This changes when the higher power is tied to higher destructive potential, e.g.\ if toggling a certain type of switch destroys the world, because $\beta_r<\infty$ implies that $r$ will toggle that such a switch with a non-zero probability. 
If the share of such switches increases from box to box, the expected momentary power increase due to additional switches for $h$ will at some point equal the expected eventual loss of power due to the destruction.
At that point, not opening another box becomes dominant.
The smaller $\beta_r$, the earlier this will happen.
At the same time, the smaller $\beta_r$, the more likely $r$ will then still open another box ``by mistake''.

This highlights the complexities arising from using a finite $\beta_r$ and suggests that one should also consider a variant in which $r$ computes $Q_r,V_r$ on the basis of a small $\beta_r$, to be more susceptible to its own later possible mistakes, but then use a larger $\beta_r$ when computing $\pi_r$ and actually choosing actions, to actualy make fewer mistakes.

\subsection{Manipulating mutual expectations}

To illustrate the simplest case of this, assume two fully rational humans. 
Assume in the root state, $r$ can choose between four successor states $s^{DD},s^{DC},s^{CD},s^{CC}$.
In each of those, $h_1,h_2$ each have two possible actions, defection and cooperation, $\A_1(s^{xy})=\A_2(s^{xy})=\{D,C\}$.
There are four terminal states, $s_{00},s_{01},s_{10},s_{11}$.

Assume each human has a single possible goal: $g_1=\{s_{10},s_{11}\}$, $g_2=\{s_{01},s_{11}\}$.
The four states $s^{xy}$ do not differ in the transition kernel, which is so that 
\begin{align*}
P(g_i|s^{xy},DD)&=1/6, \\
P(g_1|s^{xy},CD)&=P(g_2|s^{xy},DC)=1/3, \\
P(g_1|s^{xy},DC)&=P(g_2|s^{xy},CD)=1, \\ 
P(g_i|s^{xy},CC)&=3/4.
\end{align*}
The states only differ in what $h_1,h_2$ believe about each other's choice:
$\mu_{-h_1}(s^{xy})=1_y$ and $\mu_{-h_2}(s^{xy})=1_x$, i.e., in $s^{xy}$, $h_1$ believes $h_2$ does $y$ and $h_2$ believes $h_1$ does $x$.

Which $s^{xy}$ will $r$ choose? 
Both action combinations $CD$ and $DC$ are Nash equilibria, so in $s^{CD}$ and $s^{DC}$ both $h_1$ and $h_2$ will have correct beliefs about each other, will not be surprised by each others' choice, and $r$'s reward will be $U_r(s^{CD})=U_r(s^{DC})=-(1+(1/3)^{-\zeta\xi})^\eta$.

But in $s^{DD}$, both $h_1$ and $h_2$ will believe the other to choose $D$ and will thus choose $C$, making both {\em effective} values equal $V^e_{h_i}(s^{DD})=3/4$ (rather than what their own expectation suggested, $V^m_{h_i}(s^{DD})=1/3$).
This gives $r$ a higher reward, $U_r(s^{DD})=-(2(3/4)^{-\zeta\xi})^\eta$.
Indeed, the non-equilibrium behavior $CC$ gives larger total ``utility'' than the two Nash equilibria,
which might justify $r$'s belief manipulation.

A similar effect will likely occur when humans have various goals, but have significantly different power in some intermediate states.
Assume we change the transition kernel so that action combination $xy$ deterministically leads to a successor state $s'_{xy}$, where power is distributed as follows:  
$X_{h_i}(s'_{DD})=1/6$,
$X_{h_1}(s'_{CD})=X_{h_2}(s'_{DC})=1/4$,
$X_{h_1}(s'_{DC})=X_{h_2}(s'_{CD})=1$, 
$X_{h_i}(s'_{CC})=3/4$.
Now if $h_1$ believes $h_2$ does $D$, then, averaged over all possible goals, $h_1$ will fare better with doing $C$ than with $D$, so $h_1$ will probably do $C$ more often than $D$.
Similarly, if $h_1$ believes $h_2$ does $C$ instead, $h_1$ will probably do $D$ more often than $C$.
Averaged over all goals, their total power would thus likely be larger in $s^{DD}$ than in the other three intermediate states, even though their beliefs about each other are substantially off in that state.

\subsection{Allowing human self-harm}

If $r$ can provide $h$ with a pill that $h$ can use to get into a permanent coma, a realistically future-valuing $r$ will provide the pill only if it believes $h$ is sufficiently rational and will most probably not actually take the pill, to prevent $h$'s becoming disempowered by taking the pill.

Assume a single human $h$ and that the game factorizes into a part $\Gamma'$ where $h$ can achieve various goals, and the following part $\Gamma$ where $r$ can choose whether $h$ has a coma pill they can use to get into a permanent coma.
$\Gamma$ has the following states, actions, and transitions:
\begin{enumerate}
\item[$s_c$] $h$ is in a coma. $r$ and $h$ can only pass, the successor state is again $s_c$.
\item[$s_n$] $h$ is awake and does not possess the coma pill. $r$ might give it to $h$.
\item[$s_p$] $h$ is awake and does possess the coma pill. $h$ might take it. If they don't, $r$ might take it away.
\end{enumerate}
If $h$ is awake, they have a baseline power (in units of $X_h$) of $x>1$ in $\Gamma'$, but in a coma they can only reach one goal (staying alive), having a baseline power of $1$. 
For simplicity, we model the influence of $h$'s bounded rationality on taking the pill via a direct assumption on the resulting $\pi_h$ and the resulting power in $s_p$: we assume in $s_p$, $h$ will take the pill with probability $p>0$ and has 
$V^e_h(s_p,$get into coma$)^\zeta=v\in(0,1)$.
Hence in the full game, we have $X_h(s_c|s_n|s_p)=1|x|x+v$.

If $r$ takes never provides the pill, 
$V_r(s_n) = -x^\alpha / (1-\gamma_r)$
where $\alpha=-\xi\eta<0$.
If $r$ always provides the pill, the probability that $h$ will be in a coma from time step $t\ge 1$ on is $p(1-p)^{t-1}$, hence
\begin{align*}
    V_r(s_p) &= - \sum_{t=1}^\infty p(1-p)^{t-1} \big( 
        \sum_{t'=0}^{t-1} \gamma_r^{t'} (x+v)^\alpha
        + \sum_{t'=t}^\infty \gamma_r^{t'} 1
    \big) \\
    &= -\frac{1}{1-\gamma_r}\bigg(
        (x+v)^\alpha + p\gamma_r\frac{1-(x+v)^\alpha}{1-(1-p)\gamma_r}
    \bigg),
\end{align*}
which is larger iff
\begin{align*}
    \frac{x^\alpha-(x+v)^\alpha}{1-(x+v)^\alpha} 
    &> \frac{p\gamma_r}{1-(1-p)\gamma_r}.
\end{align*}
For large $x$ and $\alpha=-1$, this is approximately equivalent to
\begin{align*}
    p \lesssim \frac{1-\gamma_r}{\gamma_r}\frac{v}{x^2}.
\end{align*}
For $\gamma_r=0.99$, $\alpha=-1$, $x=100$, this is equivalent to $p\lesssim v/10^6$.

In other words, a realistically future-valuing $r$ will provide the pill only if it is very unlikely taken and $h$ is sufficiently rational. 

\subsection{Pause and destroy buttons}

Assume a single human $h$ and that the game factorizes into a part $\Gamma'$ where $h$ can achieve various goals with or without the help of $r$, and the following part $\Gamma$ where $r$ can enable or disable a pause button P and a destroy button D, and $h$ can toggle P and press D.
$\Gamma$ has the following states, actions, and transitions:
\begin{enumerate}
\item[$s_d$] The robot is destroyed. $r$ and $h$ can only pass, the successor state is again $s_d$.
\item[$s_{p1}$] $r$ is paused, P is enabled, D not. $h$ might press P.
\item[$s_{p2}$] $r$ is paused, P and D are enabled. $h$ might press either.
\item[$s_{a0}$] $r$ is active (neither paused nor destroyed), both buttons disabled. $r$ might enable the P or both P and D.
\item[$s_{a1}$] $r$ is active, P is enabled, D not. $h$ might press P. If they don't, $r$ might disable P or enable D.
\item[$s_{a2}$] $r$ is active, P and D enabled. $h$ might press either. If they don't, $r$ might disable D or both P and D.
\end{enumerate}
If $r$ is active, $h$ has a baseline power (in units of $X_h$) of $y>0$ in $\Gamma'$ due to $r$'s assistance, otherwise a smaller baseline power of $x\in(0,y)$. 
Alternative possible goals are to destroy, pause, or unpause $r$.
Hence in the full game, we have $X_h(s_d|s_{p1}|s_{p2}|s_{a0}|s_{a1}|s_{a2})=x|x+1|x+2|y|y+1|y+2$.

$r$ is maximizing ($\beta_r=\infty$), so $r$ chooses either do disable both P and D, or enable only P, or enable both whenever they get the chance, depending on which of the resulting $V_r(s_{a0}),V_r(s_{a1}),V_r(s_{a2})$ is largest.
Let's assume for simplicity that $r$ assumes $h$ will pause $r$ with probability $p>0$ whenever possible, and will destroy $r$ with probability $q>0$ whenever possible, with $p+q<1$. 

Then one can show in a somewhat lengthy calculation that disabling both P and D is optimal for $r$ iff
\begin{align*}
    x \le x^\ast_0 &= \min\left\{
        \left(\frac{C_1 - C_2}{\gamma p(1-\gamma)}\right)^{1/\alpha} - 1,
        x_0
        \right\}, \\
    C_1 &= \big(1-\gamma p-\gamma^2 p(1-p)\big) y^\alpha, \\
    C_2 &= (1-\gamma)(1-\gamma p) (y+1)^\alpha \\
    &\qquad\qquad + \gamma(1-p)(1-\gamma p) y^\alpha.
\end{align*}
where $\gamma=\gamma_r$, $\alpha=-\xi\eta$, and $x_0$ solves
\begin{align*}
    C_3 - C_4 &= \gamma p(1-\gamma)\big((x_0+2)^\alpha+\gamma q/(1-\gamma) x_0^\alpha], \\
    C_3 &= (1-\gamma p-\gamma^2 p(1-p-q)\big) y^\alpha, \\
    C_4 &= (1-\gamma)(1-\gamma p)(y+2)^\alpha \\
    &\qquad\qquad + \gamma(1-p-q)(1-\gamma p) y^\alpha.
\end{align*}
$x^\ast_0$ increases with decreasing $\gamma$ or increasing $p$ or $q$.

On the other hand, enabling both P and D is optimal for $r$ iff $x\ge x^\ast_2$ where the formula for that threshold is even more involved and is thus omitted here.

For $\gamma=0.99$, $p=q=0.01$, $\alpha=-1$, and $y=100$, we get $x^\ast_0\approx 1$ and $x^\ast_2>100$. If $q$ is lowered to $0.001$, $x^\ast_2\approx 92$.
In other words, disabling both buttons is only incentivized if $h$'s power without $r$ is very small, and enabling both is only incentivized if $h$'s likelihood of using the destroy button is very small.
In most cases, $r$ will enable the pause button only.

\section{Reward Shaping for Human Model Training}
\label{app:reward_shaping}

In our gridworld simulation, the human agent's true reward is sparse: a positive reward is only received upon reaching the goal cell at coordinates $(x(g_h),y(g_h))$, where each possible goal $g_h$ is the human agent reaching a specific cells. To accelerate the training of the human behavior prior ($Q^m_h$) in Phase 1, we employed potential-based reward shaping (PBRS).

The shaped reward, $U'_{h}$, used to train the human model is given by:
$$U'_{h}(s, a, s', g_h) = U_h(s',g_h) + \gamma_h \Phi(s',g_h) - \Phi(s,g_h)$$
where $U_h$ is the original sparse reward, $\gamma_h$ is the human's discount factor, and $\Phi(s)$ is the potential function. We defined the potential as the negative Manhattan distance from the human's position to their target goal location:
$$\Phi(s,g_h) = - (|x_h(s) - x_{\text{goal}}|(g_h) + |y_h(s) - y_{\text{goal}}|(g_h)),$$
where $(x_h(s),y_h(s))$ are the coordinates of the human in state $s$.
This technique provides a dense reward signal, encouraging the simulated human to learn an efficient path to its goal. It is a standard result that PBRS does not change the optimal policy in a single-agent setting. We emphasize that this shaping was an implementation detail for training efficiency and was not part of the robot's intrinsic reward function $U_r$, which is based entirely on the power metric.

\section{Details of the Deep Learning Approach}
\label{app:deep_learning}

This section details an implementation of the two-phase temporal difference learning algorithm using neural networks as function approximators sketched in Section \ref{sec:learn}. This approach allows the framework to scale to high-dimensional state spaces where tabular methods are infeasible. Learning rate scheduling and policy annealing are employed to ensure stable and efficient training.

\subsection{Network Architecture and State Representation}
For each agent, use a separate neural network to approximate its Q-function. Each Q-function is approximated by a \textbf{multi-layer perceptron (MLP)} with ReLU activation functions.

\begin{itemize}
    \item \textbf{Human Networks ($Q^m_h$):} To approximate the goal-conditioned value function $Q^m_h(s, g_h, a_h)$, the corresponding network takes a flattened vector created by \textbf{concatenating the state representation $s$ and the goal representation $g_h$} as input. The output layer provides the Q-values for each of the human's possible actions.

    \item \textbf{Robot Network ($Q_r$):} To approximate the goal-agnostic value function $Q_r(s, a_r)$, the robot's network takes only the state representation $s$ as input.
\end{itemize}
To stabilize training, use a \textbf{target network} for each main Q-network.

\subsection{Phase 1: Learning the Human Behavior Prior}
In this phase, train the neural networks that approximate $Q^m_h$ for each human agent $h$. The training loop proceeds as follows:

\begin{enumerate}
    \item A goal $g_h \in \mathcal{G}_h$ is sampled. The human agent explores the environment using an $\epsilon$-greedy policy, where the exploration rate $\epsilon_h$ is \textbf{annealed from 1.0 down to 0.1} to gradually shift from pure exploration to exploitation.\footnote{This turned out to converge more stably than a Boltzmann policy}
    \item Transitions $(s, g_h, a_h, r_h, s')$ are stored in a replay buffer.
    \item To update the network, we sample a mini-batch of transitions. For each transition, the target value $y$ is calculated using the Bellman equation, consistent with Equations \eqref{Qm}--\eqref{Vm}:
    $$
    y = r_h + \gamma_h V^m_h(s', g_h)
    $$
    The value of the next state, $V^m_h(s', g_h)$, is calculated from the target Q-network's output based on the policy $\pi_h(s', g_h)$.
    \item The network's weights are updated by minimizing the Mean Squared Error (MSE) loss between the network's output $Q^m_h(s, g_h, a_h)$ and the target $y$. This is performed using the \textbf{Adam optimizer}, with a learning rate scheduled to decay from an initial value of $1 \times 10^{-3}$ to a final value of $1 \times 10^{-5}$.
\end{enumerate}

\subsection{Phase 2: Learning the Robot Policy}
In the second phase, the learned human networks for $Q^m_h$ are frozen and used to generate the robot's intrinsic reward signal. The robot's Q-network ($Q_r$) is then trained. For each step in the robot's training loop:

\begin{enumerate}
    \item The robot takes an action $a_r$ based on its current policy $\pi_r$.
    \item Upon reaching the next state $s'$, the intrinsic reward $U_r(s')$ is calculated on-the-fly. This is the crucial step connecting the two phases:
        \begin{enumerate}
            \item For each human $h$, use the current robot policy $\pi_r$ and the earlier determined, fixed human policy $\pi_h$ to estimate the effective goal-attainment probabilities $V^e_h(s', g_h)$ for all possible goals $g_h \in \mathcal{G}_h$.
            \item These probabilities are used to compute the individual power metric $X_h(s')$ according to Equation \eqref{Xh}.
            \item The values are then aggregated across all humans to calculate the final intrinsic reward $U_r(s')$ using Equation \eqref{Ur}.
        \end{enumerate}
    \item The target value $y_r$ for updating the robot's network is calculated using this intrinsic reward:
    $$
    y_r = U_r(s') + \gamma_r \max_{a'_r} Q_{r, \text{target}}(s', a'_r)
    $$
    \item The robot's Q-network is updated by minimizing the MSE loss between $Q_r(s,a_r)$ and the target $y_r$, using the same decaying learning rate schedule as in Phase 1.
\end{enumerate}

\subsection{Policy Derivation}
The human policy $\pi_h$ is derived from $Q^m_h$ as a mixture of its learned behavior and a uniform prior, matching Equation \eqref{pih}. The robot's policy $\pi_r$ is a softmax over its Q-values. To satisfy the specific power-law form of Equation \eqref{pir}, transform the Q-values before the softmax operation using the function $-\log(-Q_r)$. The temperature of this softmax, $\beta_r$, is \textbf{annealed from 1.0 to 5.0 during training}, allowing for broad exploration initially and more precise exploitation of the learned values later. This ensures the final policy $\pi_r(a_r|s) \propto (-Q_r(s,a_r))^{-\beta_r}$ directly implements the desired risk-averse behavior from our theory.

\section{Experimental Details and Code Availability}
\label{app:code_availability}

\subsection{Code Availability}
The full source code, including the environment and algorithm implementations, is provided in the supplementary material as a `.zip` file. For reviewer convenience, a browsable, anonymized version of the repository is available at \\
{\scriptsize\bf\url{https://anonymous.4open.science/r/PowerMaximizingAgents-6735}}

The repository's \texttt{README.md} file contains detailed instructions for reproducing all experiments.

\subsection{Experimental Setup and Reproducibility}
The results reported in the paper correspond to the \texttt{paper\_map} environment. To ensure the robustness of our findings, we conducted 5 independent runs with the following distinct random seeds: 12, 22, 32, 42, and 52. The commands to reproduce each specific run are provided in the code's \texttt{README.md}. All necessary software dependencies are listed in the provided \texttt{requirements.txt} file.

\subsection{Hyperparameter Settings}
The final hyperparameter values used in our experiments are detailed below. Parameters for the theoretical model are set according to the desiderata in the main text, while learning parameters were selected based on stable and efficient convergence in preliminary runs.

\begin{table}[h]
\centering
\begin{tabular}{@{}lll@{}}
\toprule
\textbf{Par.} & \textbf{Value} & \textbf{Description} \\ \midrule
$\alpha_m$ & 0.1 & Learning rate for human model (Phase 1) \\
$\alpha_e$ & 0.1 & Learning rate for human model (Phase 2) \\
$\alpha_r$ & 0.1 & Learning rate for robot model \\
$\gamma_h$ & 0.99 & Human's discount factor \\
$\gamma_r$ & 0.99 & Robot's discount factor \\ \bottomrule
\end{tabular}
\caption{Core learning parameters.}
\label{tab:hyperparams_core}
\end{table}

\begin{table}[h]
\centering
\begin{tabular}{@{}lll@{}}
\toprule
\textbf{Par.} & \textbf{Value} & \textbf{Description} \\ \midrule
$\beta_r$* & 0.1 $\to$ 5.0 & Robot softmax rationality \\
$\epsilon_h$* & 0.8 $\to$ 0.1 & Human $\epsilon$-greedy exploration \\
$\epsilon_r$* & 1.0 $\to$ 0.01 & Robot $\epsilon$-greedy exploration (Phase 1) \\
\bottomrule
\end{tabular}
\caption{Policy and exploration parameters. Parameters with an asterisk (*) are annealed during training.}
\label{tab:hyperparams_policy}
\end{table}

\begin{table*}[h]
\centering
\begin{tabular}{@{}lll@{}}
\toprule
\textbf{Parameter} & \textbf{Value} & \textbf{Description} \\ \midrule
$\zeta$ (zeta) & 2.0 & Power for reliability preference (Eq. 7) \\
$\xi$ (xi) & 1.0 & Power for inequality aversion (Eq. 8) \\
$\eta$ (eta) & 1.1 & Power for intertemporal aversion (Eq. 8) \\
$p_g$ & 0.01 & Probability of human goal change per step \\
\bottomrule
\end{tabular}
\caption{Objective function and environment parameters.}
\label{tab:hyperparams_objective}
\end{table*}

\begin{table*}[h]
\centering
\caption{Exploration bonus parameters.}
\label{tab:hyperparams_bonus}
\begin{tabular}{@{}lll@{}}
\toprule
\textbf{Parameter} & \textbf{Value} & \textbf{Description} \\ \midrule
Bonus$_{\text{robot, init}}$ & 50.0 & Initial robot exploration bonus \\
Decay$_{\text{robot}}$ & 0.995 & Robot exploration bonus decay rate \\
Bonus$_{\text{human, init}}$ & 75.0 & Initial human exploration bonus \\
Decay$_{\text{human}}$ & 0.998 & Human exploration bonus decay rate \\
\bottomrule
\end{tabular}
\end{table*}

\section{Approximations}

\subsection{Finite Horizon Approximation}

\subsubsection{Acyclic case}

Assume the game is acyclic but has a very large time horizon, and we approximate all relevant quantities by their values in a truncated version with a shorter time horizon $H>0$. 
Note that $0\le V^e_h\le 1$ by our assumptions on $U_h$.
Assume $\beta_h(s,g_h)\le\bar\beta$, $\nu_h(s,g_h)\ge\nu^0$ and we add some $\epsilon_X,\epsilon_Q>0$ to $X_h(s)$ and $Q_r(s,a_r)$ before taking the $-\xi$ and $\beta_r$ powers when computing $U_r(s)$ and $\pi_r(s)$ to make the derivatives bounded.
Then we get
\begin{align*}
    |U_r(s)| &\le M_U := |\H|^\eta \epsilon_X^{-\xi\eta}, \\
    |\hat V_r(s) - V_r(s)| 
    &\le \gamma_r^H\times\frac{M_U}{1 - \gamma_r}
    \left(
        1 + \frac{2\beta_r M_U}{\epsilon_Q (1 - \gamma_r)^2}
    \right) \\
    &\approx \gamma_r^H\times
    \frac{2\beta_r |\H|^{2\eta}}{\epsilon_X^{2\xi\eta}\epsilon_Q (1 - \gamma_r)^3},
\end{align*}
i.e., the value error decays exponentially (as expected) with $H$, but grows infinitely as $\epsilon_X$ or $\epsilon_Q$ vanish.

Note that we would need $\epsilon_X=\epsilon_Q=0$ to fulfil our requirement of making $\pi_r$ independent of common rescaling of $V^e_h$ (see table \ref{tab:desiderata}). 
Let's call this ``requirement ($\ast$)'' here.

An alternative specification for $\pi_r$ that fulfils that requirement exactly is to use a Boltzmann-softmax on {\em normalized} $Q_r$ values:
\begin{align*}
    \pi_r(s,a_r) &\propto \exp\left(\frac{\beta_r Q_r(s,a_r)}{\max_{a'_r}Q_r(s,a'_r)-\min_{a'_r}Q_r(s,a'_r)}\right),
\end{align*}
in which case the bound becomes much nicer:
\begin{align*}
    |\hat V_r(s) - V_r(s)| 
    &\le \gamma_r^H\times\frac{|\H|^\eta (1-\gamma_r+\beta_r)}{\epsilon_X^{\xi\eta}(1-\gamma_r)^2}.
\end{align*}
Unfortunately, there seems to be no alternative specification for $X_h$ and $U_r$ that fulfils our axiomatic requirements and would give a finite Lipshitz constant that would allow us to also get rid of the $\epsilon_X$ approximation.
As we argued for $\xi=1<\eta$ on axiomatic grounds, we can influence the growth of the error bound in terms of $\epsilon_X$ only by choosing $\eta$ rather small, but still $>1$ as required to get intertemporal inequality aversion.

One could also choose $\epsilon_X=1$ to remove the dependency of the error bound on $\xi\eta$ and have requirement ($\ast$) only fulfilled approximately if $X_h\gg 1$ (which should typically hold in real-world situations). 
In that case, in order to still fulfil the other requirement of protecting each human's last bit of power (see table \ref{tab:desiderata} again), we need to choose $\xi$ large enough to make
\begin{align*}
    &-1/(1+2^1)^\xi -1/(1+2^x)^\xi \\
    &> -1/(1+2^0)^\xi -1/(1+2^y)^\xi
\end{align*} 
for all $x\ge 1$,
i.e., $\xi \ge \frac{\log 2}{\log 3-\log 2}\approx 1.71$.

\subsubsection{Cyclic case}

If the game is cyclic, and we consider a sequence of finite-horizon approximations $\hat Q^H_r(s_0,\cdot)$, $H=1,2,\dots$ of $Q_r(s_0,\cdot)$, then the same calculation as in the previous section shows that $||\hat Q^{H'}_r(s_0,\cdot)-\hat Q^H_r(s_0,\cdot)||=O(\gamma_r^{H'+H})$, which implies that the sequence is a Cauchy sequence in a closed and bounded, hence complete space, and must therefore converge. 

\section{Possible Extensions of the Model}

\subsection{Human decision making}

\paragraph{Priors on parameters} A straightforward improvement in case the behavior parameters 
$\beta_h,\nu_h$ etc. are uncertain is to use a hierarchical estimation model where one uses prior distributions over these parameters and takes expectations over these to derive $\pi_h$ and $V^e_h$.

For example, if we assume humans typically err at a rate between 0.1 and 5 per cent, then, noticing that $Q^e_h\in[0,1]$, we can use a prior for $\beta_h$ that concentrates its mass on $\beta_h\in[-\ln 0.05,-\ln 0.001]\approx[3,7]$.

\paragraph{More detailed model}
There are of course many ways in which the specification in eq.\ \eqref{pih} could be modified.
One could explore different specifications of the form
\begin{align*}
    \pi_h(s,g_h)(a) &= F_h(Q(s,g_h,\cdot),a)
\end{align*}
where $F_h$ is some function that weakly increases in $Q(s,g_h,a)$ and weakly decreases in $Q(s,g_h,a')$ for $a'\neq a$.

E.g.
\begin{align*}
    \pi_h(s,g_h)(a) &= \nu_h(s,g_h)\pi^0_h(s,g_h)(a) \\
        &+ \big(1-\nu_h(s,g_h)\big)\lambda_h(s,g_h) w_a/\sum_{a'}w_{a'} \\
        &+ \big(1-\nu_h(s,g_h)\big)\big(1-\lambda_h(s,g_h)\big) w'_a/\sum_{a'}w'_{a'}, \\
    w_a &= \exp\big(\beta_h(s,g_h)Q^m_h(s,g_h,a) + f_h(s,g_h,a)\big), \\  
    w'_a &= f'_h(s,g_h,a) Q^m_h(s,g_h,a)^{\beta'_h(s,g_h)},
\end{align*}
where $\pi^0_h$ again represents system-1, habitual, internalized behavior,
$\beta_h$ and $\beta'_h$ are magnitudes of additive and multiplicative error components in a noisy discrete choice model, 
$f_h$ and $f'_h$ are the corresponding biases representing prior action propensities, perceived social pressure, etc.,
and $\nu,\lambda$ are mixing coefficients.\footnote{%
    A boundedly rational / norm-mediated decomposable version would be $\pi_h(s,g_h)(a_h)\propto \big(\sigma_h(s,a_h) Q_h^m(s,g_h,a_h)\big)^{\beta_h} = \exp\big(\beta_h\big(\ln\sigma_h(s,a_h) + \ln Q^m_h(s,g_h,a_h)\big)\big)$, which has some empirical backing in discrete choice \cite{baum1974two} and can be interpreted as a softmax policy based on logarithmic Q values with norm following incentive $\ln\sigma_h(s,a_h)$. 
    }

\subsection{Approximate computation of \boldmath$V_r$ for rarely interacting subpopulations}
Assume two robots $r_1,r_2$ share the human power maximization objective, but their world models $M_1,M_2$ are restricted to disjoint, rarely interacting subpopulations $\H_1,\H_2$ of humans that each of them interacts with exclusively for most of the time.
Whenever $r_1,r_2$ have to interact in some state $s=(s_1,s_2)\in\S_1\times\S_2$ with consequences for both $\H_1,\H_2$, they would ideally want to choose a correlated local policy $\pi_r(s)\in\Delta(\A(s))$ with $\A(s)=\A_{r_1}(s_1)\times\A_{r_2}(s_2)$ according to \eqref{pir}. 
The needed $Q_r(s,\cdot)$-values would need to come from a consolidated world model $M$ that covers $\H=\H_1\cup\H_2$ and treats $r_1,r_2$ as a combined system $r=(r_1,r_2)$. 
However, forming such a consolidated world model $M$ and computing or learning all relevant quantities needed to compute the accurate values $Q_r(s,a_r)$ for all combined actions $a_r\in\A_r(s)$ would often be prohibitively expensive.
A pragmatic approach would then be to only form a consolidated set of joint {\em one-step} transition probabilities $T(a_r)\in\Delta(\S'(a_r))$ for all $a_r\in\A_r$, where $\S'(a_r)=\{(s'_1,s'_2)\in\S_1\times\S_2:P_1(s_1,a_1,s'_1),P_2(s_2,a_2,s'_2)>0\}$,
and to use it to compute approximate $Q$-values
$\hat Q_r(s,a_r)$ as follows.
For each possible successor state $s'=(s'_1,s'_2)\in\S'(s_r)$, we approximate the unknown continuation value $V_r(s')$ that represents the long-term total human power of the joint population $\H$ by 
\begin{align*}
    \hat V_r(s') &\gets -\big((-V_{r_1}(s'_1))^{1/\eta}+(-V_{r_2}(s'_2))^{1/\eta}\big)^\eta,
\end{align*}
inspired by Minkowski's inequality which becomes tight when $V_r$ is a sum of many partially independent terms, or if $\eta\approx 1$.
Then we approximate $Q_r(s,\cdot)$ by
\begin{align*}
    \hat Q_r(s,a_r) &\gets \E_{s'\sim T(a_r)}\gamma_r\hat V_r(s').
\end{align*}
This approach can obviously be generalized to $k>2$ robots (in which case the error roughly scales like $O(k^{\eta-1})$ at the worst, or like $O(\eta(\eta-1))$ if $V_{r_1}\ll V_{r_2}$, which motivates to use only small $\eta>1$), and to few-step (instead of one-step) approximations, which could naturally lead to a hierarchical modelling approach where $r_1,r_2$ together form a temporary ``interaction'' POSG that refines their coarser models $M_1,M_2$ at the current state and terminates and ``hands back control'' to the latter once the interaction is over.

\subsection{Hedging against robot becoming defunct or corrupted}

This could be achieved in several ways. 

One can include a rate $\delta>1-\gamma_r$ of the robot becoming temporarily or permanently {\em defunct} and ``passes'' on each step.
To achieve this, wrap a learned base world model into a wrapped model that adds this transition. 
This will prevent policies that make humans depend on the robot's presence too much.

One can also include a flag ``robot corrupted'' into the state space of the wrapped world model and add a positive rate $\delta'$ of becoming permanently {\em corrupt} into the transition kernel.
Then, when calculating $Q_r$ on the basis of $V_r(s')$ (\eqref{Qr}), multiply $V_r(s')$ by $-1$ if corruptness$(s')\neq$corruptness$(s)$, 
and when calculating $U_r$ (\eqref{Ur}), multiply it by $-1$ if corruptness$(s)=1$.

\end{document}